\documentclass[letterpaper]{article}
\usepackage{aaai2026}
\usepackage{times}
\usepackage{helvet}
\usepackage{courier}
\usepackage[hyphens]{url}
\usepackage{graphicx}
\usepackage{natbib}
\usepackage{caption}
\usepackage[utf8]{inputenc}
\usepackage{amsmath,amsfonts,amssymb}
\usepackage{amsthm}
\usepackage{aliascnt}
\usepackage{enumerate}
\usepackage{multirow}
\usepackage{hhline}
\usepackage{tikz}
\usetikzlibrary{arrows.meta,positioning,calc,patterns}
\usepackage{mdframed}
\usepackage{tabularx}
\usepackage{tabulary}
\usepackage{booktabs}
\usepackage{nicefrac}
\usepackage{manyfoot}
\usepackage[protrusion=true,expansion=false]{microtype}
\usepackage[noabbrev,capitalise]{cleveref}
\DeclareMathOperator{\argmax}{arg\,max}
\DeclareMathOperator{\Argmax}{Arg\,max}
\DeclareMathOperator{\argmin}{arg\,min}
\DeclareMathOperator{\co}{co}
\DeclareMathOperator{\diag}{diag}
\DeclareMathOperator{\dist}{dist}

\makeatletter
\@ifundefined{th@thmstyleone}{%
  \newtheoremstyle{thmstyleone}{18pt plus2pt minus1pt}{18pt plus2pt minus1pt}{\itshape}{0pt}{\bfseries}{}{.5em}{\thmname{#1}\thmnumber{\@ifnotempty{#1}{ }\@upn{#2}}\thmnote{ {\the\thm@notefont(#3)}}}%
}{}
\@ifundefined{th@thmstyletwo}{%
  \newtheoremstyle{thmstyletwo}{18pt plus2pt minus1pt}{18pt plus2pt minus1pt}{\normalfont}{0pt}{\itshape}{}{.5em}{\thmname{#1}\thmnumber{\@ifnotempty{#1}{ }{#2}}\thmnote{ {\the\thm@notefont(#3)}}}%
}{}
\@ifundefined{th@thmstylethree}{%
  \newtheoremstyle{thmstylethree}{18pt plus2pt minus1pt}{18pt plus2pt minus1pt}{\normalfont}{0pt}{\bfseries}{}{.5em}{\thmname{#1}\thmnumber{\@ifnotempty{#1}{ }\@upn{#2}}\thmnote{ {\the\thm@notefont(#3)}}}%
}{}
\makeatother

\theoremstyle{thmstyleone}
\newtheorem{theorem}{Theorem}
\newaliascnt{claim}{theorem}

\aliascntresetthe{claim}
\newaliascnt{lemma}{theorem}
\newtheorem{lemma}[lemma]{Lemma}
\aliascntresetthe{lemma}
\newaliascnt{corollary}{theorem}
\newtheorem{corollary}[corollary]{Corollary}
\aliascntresetthe{corollary}
\newaliascnt{proposition}{theorem}

\aliascntresetthe{proposition}

\theoremstyle{thmstylethree}
\newaliascnt{example}{theorem}
\newtheorem{example}[example]{Example}
\aliascntresetthe{example}
\newaliascnt{assumption}{theorem}
\newtheorem{assumption}[assumption]{Assumption}
\aliascntresetthe{assumption}
\newaliascnt{remark}{theorem}

\aliascntresetthe{remark}
\newaliascnt{problem}{theorem}

\aliascntresetthe{problem}
\newaliascnt{definition}{theorem}
\newtheorem{definition}[definition]{Definition}
\aliascntresetthe{definition}

\crefname{theorem}{Theorem}{Theorems}
\Crefname{theorem}{Theorem}{Theorems}
\crefname{claim}{Claim}{Claims}
\Crefname{claim}{Claim}{Claims}
\crefname{lemma}{Lemma}{Lemmas}
\Crefname{lemma}{Lemma}{Lemmas}
\crefname{corollary}{Corollary}{Corollaries}
\Crefname{corollary}{Corollary}{Corollaries}
\crefname{proposition}{Proposition}{Propositions}
\Crefname{proposition}{Proposition}{Propositions}
\crefname{example}{Example}{Examples}
\Crefname{example}{Example}{Examples}
\crefname{assumption}{Assumption}{Assumptions}
\Crefname{assumption}{Assumption}{Assumptions}
\crefname{remark}{Remark}{Remarks}
\Crefname{remark}{Remark}{Remarks}
\crefname{problem}{Problem}{Problems}
\Crefname{problem}{Problem}{Problems}
\crefname{definition}{Definition}{Definitions}
\Crefname{definition}{Definition}{Definitions}

\newmdenv[
  linewidth=0.6pt,
  linecolor=black,
  skipabove=6pt,
  skipbelow=6pt,
  innertopmargin=6pt,
  innerbottommargin=6pt,
  innerleftmargin=8pt,
  innerrightmargin=8pt
]{restatementbox}

\newcommand{\R}{\mathbb R}

\newcommand{\calS}{\mathcal S}
\newcommand{\calA}{\mathcal A}
\newcommand{\calH}{\mathcal H}

\newcommand{\calF}{\mathcal F}
\newcommand{\E}{\mathbb{E}}
\newcommand{\Prob}{\mathbb{P}}
\newcommand{\calM}{\mathcal{M}}

\newcommand{\missingfigure}[1]{%
  \fbox{%
    \begin{minipage}{0.68\linewidth}
    \centering
    Missing figure file: \texttt{#1}
    \end{minipage}%
  }%
}
\newcommand{\safeincludegraphics}[2][]{%
  \IfFileExists{#2}{\includegraphics[#1]{#2}}{%
    \IfFileExists{#2.pdf}{\includegraphics[#1]{#2}}{%
      \IfFileExists{#2.png}{\includegraphics[#1]{#2}}{%
        \IfFileExists{#2.eps}{\includegraphics[#1]{#2}}{\missingfigure{#2}}%
      }%
    }%
  }%
}

\title{Sign-Separated Asymmetric Finite-Time Error Analysis of Q-Learning}

\author{%
\small Donghwan Lee\\
\small Department of Electrical Engineering\\
\small Korea Advanced Institute of Science and Technology (KAIST)\\
\small Daejeon 34141, South Korea (email: \texttt{donghwan@kaist.ac.kr})
}

\begin{document}

\maketitle

\begin{abstract}
Q-learning is known to suffer from overestimation bias: because the Bellman update maximizes noisy or imperfect action-value estimates, positive errors can be selected and propagated, causing learned values to exceed the true optimal values. This bias can slow learning, degrade policy quality, and make value estimates unreliable. Although the convergence of Q-learning has been studied extensively, convergence theory that explicitly reflects this overestimation mechanism remains limited. This paper studies the asymmetric convergence behavior of Q-learning induced by overestimation bias. We decompose the Q-learning error into its componentwise positive and negative parts and derive separate finite-time rates for the two components. The resulting certificates can assign a slower exponential envelope to the positive component than to the negative component. This rate separation provides indirect theoretical evidence for max-induced overestimation: positive errors can be amplified through the maximization step, whereas negative errors admit a sharper comparison with an optimal-policy system. The separation is a difference between upper bounds, so it need not hold for every realized Q-learning trajectory. Nevertheless, we construct examples in which the predicted asymmetry appears in the actual trajectory. The analysis gives deterministic and stochastic constant-step-size bounds and clarifies how overestimation enters the switching-system dynamics of Q-learning.
\end{abstract}

\section{Introduction}

Q-learning~\citep{watkins1992q} is a foundational algorithm in reinforcement
learning (RL)~\citep{sutton1998reinforcement} for solving Markov decision
processes (MDPs)~\citep{puterman1994markov} with unknown transition kernels. Its update includes a
maximization over estimated action values. When these estimates are noisy or
imperfect, the maximum can select positive errors and propagate them through the
Bellman update, producing the well-known overestimation bias~\citep{thrun1993issues,vanhasselt2010double}. This bias can slow
learning, degrade the quality of greedy policies, and make learned value
estimates less reliable. Although Q-learning convergence has been studied
extensively~\citep{tsitsiklis1994asynchronous,jaakkola1994convergence,
borkar2000ode,lee2020unified,szepesvari1998asymptotic,kearns1999finite,even2003learning,
beck2012error,wainwright2019stochastic,qu2020finite,li2020sample,
chen2024lyapunov,lee2023discrete,lee2024final,lim2024diminishing,lee2026lyapunov}, finite-time theory that explicitly reflects this max-induced
asymmetry remains limited.

As a  motivating instance, let us consider Sutton's maximization-bias example in~\cite{sutton1998reinforcement}. Starting from a single fixed initial Q-table, \cref{fig:sutton-maxbias-relative} plots the relative sign-separated errors averaged over ten runs. The positive-direction error remains larger for longer, illustrating the empirical asymmetry.
\begin{center}
\includegraphics[width=0.95\linewidth]{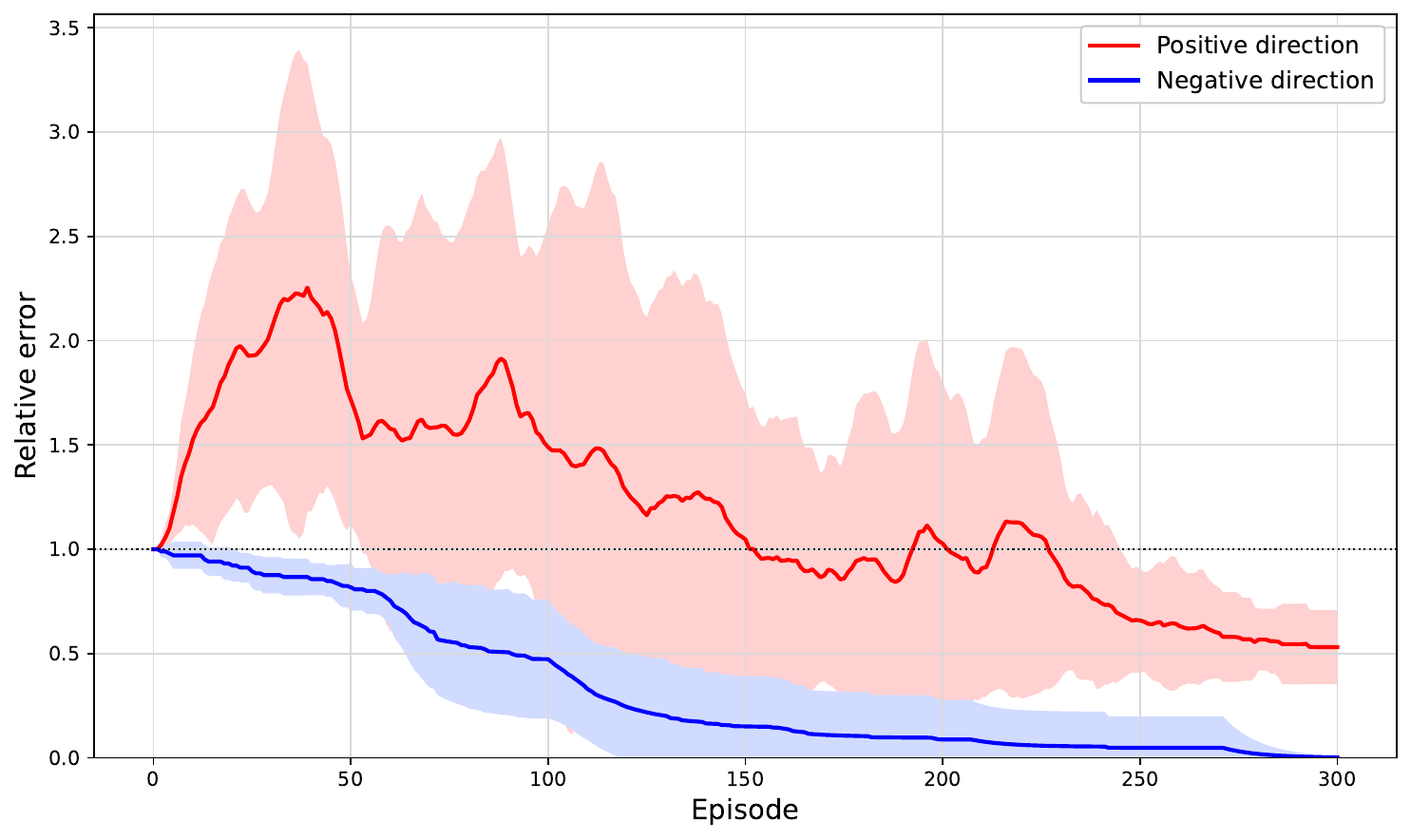}
\captionof{figure}{Sutton maximization-bias example with sign-separated relative errors.}
\label{fig:sutton-maxbias-relative}
\end{center}
This paper studies the asymmetric convergence behavior of Q-learning induced by overestimation bias. We decompose the Q-learning error into its componentwise positive and negative parts and derive separate finite-time rates for the two components. The resulting certificates can assign a slower exponential envelope to the positive component than to the negative component. This rate separation provides indirect theoretical evidence for max-induced overestimation: positive errors can be amplified through the maximization step, whereas negative errors admit a sharper comparison with an optimal-policy system. The separation is a difference between upper bounds, so it need not hold for every realized Q-learning trajectory. Nevertheless, we construct examples in which the predicted asymmetry appears in the actual trajectory. The analysis gives deterministic and stochastic constant-step-size bounds and clarifies how overestimation enters the dynamics of Q-learning.

Our analysis is based on a switching linear system representation of Q-learning~\citet{lee2026lyapunov}, together with the direct switching technique recently proposed by~\citet{lee2026lyapunov}. In a switched linear system~\citep{liberzon2003switching,lin2009stability,shorten2007stability}, the update matrix is selected from a prescribed family of matrices and may change over time. This representation allows Q-learning convergence to be interpreted in the language of switched linear dynamics, particularly through the joint spectral radius~\citep{rota1960rho,tsitsiklis1997lyapunov,blondel2005rho,jungers2009rho}. More specifically, the Q-learning error is decomposed into its positive and negative parts, and we construct comparison systems that bound the two components. The negative-part comparison system is a simpler linear time-invariant system, whereas the positive-part comparison system remains a switching linear system over the set of greedy policies. We show that the negative comparison system can have a faster certified convergence rate than the positive comparison system. To keep the presentation clear and focused, we study constant step-size Q-learning under an i.i.d. observation model. We also treat the simpler deterministic counterpart, but defer it to the appendix because of space constraints. The method of~\citet{lee2026lyapunov} can be used to extend the analysis to Markovian observations, but the i.i.d. setting keeps the sign-separated switching-system argument more transparent. Overall, the analysis provides a new switching system interpretation of finite-time Q-learning behavior and suggests how this perspective can yield further insight into max-induced asymmetry.

\section{Preliminaries}\label{sec:preliminaries}

\subsection{Notation}\label{subsec:notation}

We use the following notation. The symbols \({\mathbb R}\),
\({\mathbb R}^n\), and \({\mathbb R}^{n\times m}\) denote the set of real
numbers, the \(n\)-dimensional Euclidean space, and the set of \(n\times m\)
real matrices, respectively. For a matrix \(A\), \(A^\top\) denotes its
transpose. The identity matrix with appropriate dimensions is denoted by \(I\).
For a finite set \(\mathcal S\), its cardinality is denoted by
\(|\mathcal S|\). The Kronecker product of \(A\) and \(B\) is denoted by
\(A\otimes B\). For a square matrix \(A\), \(\rho(A)\) denotes its spectral
radius.

For a set \(\mathcal Y \subset \mathbb R^n\) and a vector
\(x\in\mathbb R^n\), \(\dist_\infty(x,\mathcal Y)\) denotes the infinity-norm
distance from \(x\) to \(\mathcal Y\):
$\dist_\infty(x,\mathcal Y):=\inf_{y\in\mathcal Y}\|x-y\|_\infty$.
We write \(\Delta_{|{\cal A}|}\) for the probability simplex over a finite
action set \({\cal A}\): $\Delta_{|{\cal A}|}
:=\left\{
p\in\mathbb R^{|{\cal A}|}:\ p_i\ge 0,\ \sum_{i=1}^{{|{\cal A}|}}p_i=1
\right\}$. Throughout the paper, \(\Argmax\) denotes the set-valued maximizer, while
\(\argmax\) denotes a fixed tie-broken single-valued maximizer.

All vector inequalities are understood componentwise unless otherwise stated.
For a scalar \(x\), define
\[
x^+:=\max\{x,0\},
  \qquad
  x^-:=\max\{-x,0\}.
\]
Thus, \(x^+\) is the positive part of \(x\), while \(x^-\) is the magnitude of
the negative part of \(x\). For a vector \(x\), \(x^+\) and \(x^-\) are
defined componentwise, and therefore
\[
\begin{gathered}
x=x^+-x^-,\qquad |x|=x^++x^-,\\
x^+\ge0,\qquad x^-\ge0.
\end{gathered}
\]
For a finite matrix family \(\calH=\{A_1,\ldots,A_N\}\), the notation
\(\co(\calH)\) denotes the convex hull $\co(\calH)  := \left\{
  \sum_{i=1}^N\lambda_iA_i: \lambda_i\ge0,
  \sum_{i=1}^N\lambda_i=1 \right\}$.

\subsection{Switching Systems}\label{subsec:switching-system}

Let us consider the discrete-time switched linear system~\citep{liberzon2003switching,lin2009stability,shorten2007stability}
\[
z_{k+1}=A_{\sigma_k}z_k+\xi_k,
  \qquad k\in\{0,1,2,\ldots\},
\]
where each index \(i\in\{1,2,\ldots,M\}\) is called a mode and corresponds to
one matrix \(A_i\). The sequence \(\sigma_k\in\{1,2,\ldots,M\}\) is the
switching signal; it specifies which mode is active at time \(k\). Equivalently, saying
that mode \(\sigma_k=i\) is active means that the update from \(z_k\) to
\(z_{k+1}\) uses the dynamics matrix \(A_i\). The prescribed set of all possible
mode matrices \(\calH:=\{A_1,A_2,\ldots,A_M\}\) is called the switching family,
and \(\xi_k\) is an additive disturbance. In this paper, a mode denotes the
currently applied dynamics matrix; in the Q-learning applications below, modes
are induced by policy selectors. When \(\xi_k=0\), the deterministic part reduces to
\[
z_{k+1}=A_{\sigma_k}z_k,
  \qquad k\in\{0,1,2,\ldots\}.
\]
If the switching family has a single element, say \(\calH=\{H\}\), then
there is no genuine mode variation and the switched system reduces to the usual
linear time-invariant (LTI) recursion
\[
z_{k+1}=Hz_k+\xi_k,
\]
or, in the disturbance-free case, \(z_{k+1}=Hz_k\). Thus, LTI systems are
included as the singleton-family special case of switching systems.
The worst-case exponential rate of a switched linear family is characterized by
the joint spectral radius (JSR)~\citep{tsitsiklis1997lyapunov,rota1960rho,blondel2005rho,jungers2009rho}, defined as follows.
\begin{definition}\label{def:rho}
For a bounded set of matrices \(\calH\subset\R^{m\times m}\), its JSR is denoted by
\[
\rho(\calH)
  :=
  \lim_{k\to\infty}
  \sup_{A_1,\ldots,A_k\in\calH}
  \|A_k\cdots A_1\|^{1/k},
\]
where the value is independent of the chosen submultiplicative norm.  When
\(\calH\) is finite, the supremum for each fixed product length is a maximum
over all products generated by matrices in \(\calH\).  If
\(\calH=\{H\}\) consists of a single matrix, then this definition reduces to
the usual spectral radius:
\[
\rho(\{H\})=\lim_{k\to\infty}\|H^k\|^{1/k}=\rho(H).
\]
\end{definition}

\begin{lemma}\label{lem:product-growth-bound}
Let \(\calH\subset\R^{m\times m}\) be finite. For any
\(\beta>\rho(\calH)\), define
\[
K_\beta(\calH)
  :=
  \sup_{t\ge0}\beta^{-t}
  \max_{A_0,\ldots,A_{t-1}\in\calH}
  \|A_{t-1}\cdots A_0\|_\infty,
\]
where the product of length zero is the identity. Then
\(K_\beta(\calH)<\infty\), and every product generated by \(\calH\) satisfies
\[
\|A_{t-1}\cdots A_0\|_\infty
  \le
  K_\beta(\calH)\beta^t,
  \qquad t\ge0.
\]
\end{lemma}

\begin{proof}
By \cref{def:rho}, for any \(\beta>\rho(\calH)\) there exists an integer
\(T\) such that
\[
\max_{A_0,\ldots,A_{t-1}\in\calH}
\|A_{t-1}\cdots A_0\|_\infty^{1/t}\le \beta
\]
for every \(t\ge T\). Hence the supremum in the definition of
\(K_\beta(\calH)\) is attained over a finite set of product lengths up to a
bounded tail, and is finite. The product estimate follows immediately from the
same definition.
\end{proof}

\subsection{Markov Decision Processes}\label{subsec:mdp}

We consider an infinite-horizon discounted Markov decision process
(MDP)~\citep{puterman1994markov}, in which an agent sequentially chooses actions
to maximize cumulative discounted rewards. The state and action spaces are
finite and are denoted by \({\cal S}:=\{1,2,\ldots,|{\cal S}|\}\) and
\({\cal A}:=\{1,2,\ldots,|{\cal A}|\}\), respectively. At state
\(s\in{\cal S}\), the decision maker selects an action \(a\in{\cal A}\). The
next state \(s'\) is drawn according to \(P(s'|s,a)\), and the transition
incurs reward \(r(s,a,s')\), where
\(r:{\cal S}\times{\cal A}\times{\cal S}\to{\mathbb R}\). We write
\(r(s_k,a_k,s_{k+1})=:r_{k+1}\) for \(k\ge0\). The expected one-step reward is
\[
\begin{aligned}
R(s,a)
&:={\mathbb E}[r_{k+1}\mid s_k=s,a_k=a]\\
&=\sum_{s'\in\calS}P(s'|s,a)r(s,a,s').
\end{aligned}
\]
A deterministic policy \(\pi:{\cal S}\to{\cal A}\) maps each state \(s\) to an
action \(\pi(s)\). Throughout the paper, the discount factor satisfies
\(\gamma\in(0,1)\). Let \(\Theta\) denote the set of all admissible
deterministic policies. For a policy \(\pi\), the Q-function under \(\pi\) is defined as
\[
Q^{\pi}(s,a)={\mathbb E}\left[ \left. \sum_{k=0}^\infty {\gamma^k r_{k+1}} \right|s_0=s,a_0=a,\pi \right],
\]
for all \(s\in{\cal S}\) and \(a\in{\cal A}\). The corresponding value
function is \(V^\pi(s):=Q^\pi(s,\pi(s))\). The optimal Q-function is
\[
Q^*(s,a):=\sup_{\pi\in\Theta} Q^\pi(s,a),\qquad s\in{\cal S},\ a\in{\cal A}.
\]
A deterministic policy \(\pi^*\) is optimal if
\(Q^{\pi^*}(s,a)=Q^*(s,a)\) for all
\((s,a)\in{\cal S}\times{\cal A}\). Once \(Q^*\) is known, an optimal
tie-broken greedy policy can be recovered as
\(\pi^*(s)=\argmax_{a\in {\cal A}}Q^*(s,a)\). The corresponding optimal value
function is
\[
V^*(s):=\max_{a\in {\cal A}}Q^*(s,a).
\]
For each state \(s\in{\cal S}\), define the set of optimal greedy actions by
\[
\Phi^*(s):=\Argmax_{a\in {\cal A}}Q^*(s,a).
\]
The set of all optimal deterministic policies is then
\[
\Theta^*
:=
\{\pi\in\Theta:\ \pi(s)\in\Phi^*(s),\ \forall s\in {\cal S}\}.
\]
The definition above is equivalent to the usual set of optimal deterministic
policies in a finite discounted MDP. If \(\pi\in\Theta^*\), then \(V^*\) satisfies the Bellman evaluation equation for \(\pi\). Uniqueness of the discounted evaluation fixed point gives \(V^\pi=V^*\), and hence
\(Q^\pi=Q^*\). Conversely, if \(\pi\) is optimal, then \(V^\pi=V^*\), and
the policy-evaluation equation implies
\(V^*(s)=Q^*(s,\pi(s))\) for every state. Thus
\(\pi(s)\in\Phi^*(s)\) for all \(s\), so \(\Theta^*\) is exactly the set
of optimal deterministic policies.

\subsection{Definitions}\label{sec:definitions}

In this paper, we consider a finite discounted Markov decision process
(MDP)~\citep{puterman1994markov} with state-space
\(\calS=\{1,\ldots,|\calS|\}\), action-space
\(\calA=\{1,\ldots,|\calA|\}\), transition probability
\(P(s'\mid s,a)\), real-valued one-step reward \(r(s,a,s')\), expected
reward
\[
R(s,a):=\sum_{s'\in\calS}P(s'\mid s,a)r(s,a,s'),
\]
and discount factor \(\gamma\in(0,1)\). State-action functions are viewed as
vectors in \(\R^{|\calS||\calA|}\) using the action-block ordering
\[
(1,1),(2,1),\ldots,(|\calS|,1),(1,2),(2,2),\ldots,(|\calS|,|\calA|).
\]
All matrices and vectors indexed by state-action pairs use this ordering. For \(Q\in\R^{|\calS||\calA|}\),
\[
Q=
\begin{bmatrix}
   Q(\cdot,1)\\
   \vdots\\
   Q(\cdot,|\calA|)
\end{bmatrix},
\qquad
Q(s,a)=(e_a\otimes e_s)^\top Q,
\]
where \(e_s\in\R^{|\calS|}\) and \(e_a\in\R^{|\calA|}\) are the standard
basis vectors. Let us define the matrix
\[
\begin{gathered}
P:=
\begin{bmatrix}
P_1\\ \vdots\\ P_{|\calA|}
\end{bmatrix}
\in\R^{|\calS||\calA|\times |\calS|},\\
R:=
\begin{bmatrix}
R(\cdot,1)\\ \vdots\\ R(\cdot,|\calA|)
\end{bmatrix}
\in\R^{|\calS||\calA|}.
\end{gathered}
\]
where \(P_a=P(\cdot\mid\cdot,a)\in\R^{|\calS|\times|\calS|}\). For the
finite MDP above, let us define
\[
R_{\max}:=\max_{(s,a,s')\in\calS\times\calA\times\calS}|r(s,a,s')|.
\]
Because the state and action spaces are finite and rewards are real-valued,
\(R_{\max}<\infty\).

Let \(\Theta\) denote the set of deterministic stationary policies
\(\pi:\calS\to\calA\). For any stochastic policy
\(\mu:\calS\to\Delta_{|\calA|}\), we define
\[
\boldsymbol{\Pi}^\mu
:=
\begin{bmatrix}
\mu(1)^\top\otimes e_1^\top\\
\mu(2)^\top\otimes e_2^\top\\
\vdots\\
\mu(|\calS|)^\top\otimes e_{|\calS|}^\top
\end{bmatrix}
\in\R^{|\calS|\times |\calS||\calA|}.
\]
For a deterministic policy \(\pi\in\Theta\), we use the same notation
\(\boldsymbol{\Pi}^\pi\) by identifying \(\pi(s)\) with its one-hot encoding.
Then \(P\boldsymbol{\Pi}^\mu\in\R^{|\calS||\calA|\times|\calS||\calA|}\) is
the transition matrix of the state-action pair induced by \(\mu\). For
\(Q\in\R^{|\calS||\calA|}\), define
\[
\begin{gathered}
V_Q(s):=\max_{a\in\calA}Q(s,a),\\
V_Q:=(V_Q(1),\ldots,V_Q(|\calS|))^\top.
\end{gathered}
\]
The Bellman optimality operator is written as
\[
F(Q):=R+\gamma PV_Q.
\]
With this notation, \(Q^*\) is the unique fixed point of \(F\), and
\(V^*=V_{Q^*}\). For any \(Q\in\R^{|\calS||\calA|}\), let
\(\pi_Q(s):=\argmax_{a\in\calA}Q(s,a)\) denote the tie-broken greedy policy
with respect to \(Q\). We also use the shorthand
\[
\boldsymbol{\Pi}_Q:=\boldsymbol{\Pi}^{\pi_Q}.
\]
The advantage function at state \(s\) and action \(a\) is
\[
A^*(s,a):=V^*(s)-Q^*(s,a)\ge0.
\]
Consequently,
\[
A^*(s,a)=0
  \quad\Longleftrightarrow\quad
  a\in\Phi^*(s).
\]

\subsection{Q-Learning}\label{subsec:qlearning-error}

We consider the standard asynchronous Q-learning recursion~\citep{sutton1998reinforcement,bertsekas1996neuro}
with a constant step-size \(\alpha\) under an i.i.d. observation model. At step
\(k\), a state-action pair \((s_k,a_k)\) is sampled independently across \(k\)
according to
\[
\begin{gathered}
\Prob(s_k=s,a_k=a)=d(s,a):=p(s)b(a\mid s),\\
(s,a)\in\calS\times\calA.
\end{gathered}
\]
where \(p\) is a state-sampling distribution and \(b\) is a behavior policy.
Then \(s'_k\sim P(\cdot\mid s_k,a_k)\) is sampled independently conditional on
\((s_k,a_k)\), and the reward sample is taken as
\[
r_{k+1}:=r(s_k,a_k,s'_k),
  \qquad k\in\{0,1,2,\ldots\}.
\]
The asynchronous Q-learning update is
\[
\begin{aligned}
&Q_{k+1}(s_k,a_k)\\
&\quad=
Q_k(s_k,a_k)+\alpha\Bigl(
 r_{k+1}+\gamma\max_{u\in\calA}Q_k(s'_k,u)\\
&\qquad\qquad
 -Q_k(s_k,a_k)\Bigr),\\
&\quad k\in\{0,1,2,\ldots\}.
\end{aligned}
\]
All other coordinates \((s,a)\ne(s_k,a_k)\) remain unchanged:
\[
\begin{gathered}
Q_{k+1}(s,a)=Q_k(s,a),\\
(s,a)\ne(s_k,a_k),\qquad k\in\{0,1,2,\ldots\}.
\end{gathered}
\]
Let \(\{\calF_k\}_{k\ge0}\) be the natural filtration of this Q-learning
process,
\[
\begin{aligned}
\calF_0&:=\sigma(Q_0),\\
\calF_k&:=
  \sigma\!\Bigl(
  Q_0,
  \bigl\{(s_t,a_t,s'_t,r_{t+1}):\\
&\qquad\qquad 0\le t\le k-1\bigr\}
  \Bigr),\quad k\ge1.
\end{aligned}
\]
Then \(Q_k\) is \(\calF_k\)-measurable.  Moreover, the fresh observation
\((s_k,a_k,s'_k,r_{k+1})\) is independent of \(\calF_k\) and is revealed
between times \(k\) and \(k+1\).

Let
\[
D:=\diag(d(s,a))_{(s,a)\in\calS\times\calA}
  \in\R^{(|\calS|\,|\calA|)\times(|\calS|\,|\calA|)}.
\]
For any stochastic policy \(\mu:\calS\to\Delta_{|\calA|}\), define
\[
\mathbf A_\mu
  :=
  I-\alpha D+\alpha\gamma DP\boldsymbol{\Pi}^\mu.
\]
For the deterministic-policy family, set
\[
\begin{gathered}
\calM_\alpha:=\{\mathbf A_\pi:\pi\in\Theta\},
  \qquad
\rho_+:=\rho_\alpha^{\mathrm{dir}}:=\rho(\calM_\alpha),\\
\calM_\alpha^-:=\{\mathbf A_\pi:\pi\in\Theta^*\},
  \qquad
\rho_-:=\rho(\calM_\alpha^-).
\end{gathered}
\]
Choose once and for all
\[
\begin{gathered}
\pi_-^\star\in\argmin_{\pi\in\Theta^*}\rho(\mathbf A_\pi),
  \qquad
\mathbf A_-^\star:=\mathbf A_{\pi_-^\star},\\
\rho_-^\star:=\rho(\mathbf A_-^\star)
  =
  \min_{\pi\in\Theta^*}\rho(\mathbf A_\pi).
\end{gathered}
\]
Write \(e_{s,a}\) for the standard basis vector corresponding to coordinate
\((s,a)\), equivalently \(e_a\otimes e_s\) under the chosen Kronecker ordering.
For the sampled coordinate, define the random vector
\[
\zeta_k:=e_{s_k,a_k}
\left(
 r_{k+1}+\gamma\max_{u\in\calA}Q_k(s'_k,u)-Q_k(s_k,a_k)
\right).
\]
Since \(Q_k\) is \(\calF_k\)-measurable and the observation at time \(k\) is
independent of \(\calF_k\),
\[
\E[\zeta_k\mid\calF_k]
  =
  D(F(Q_k)-Q_k).
\]
Accordingly, the martingale-difference noise is
\begin{equation}
\label{eq:wk-def}
\begin{aligned}
w_k
:=&
e_{s_k,a_k}\Bigl(
 r_{k+1}
 +\gamma\max_{u\in\calA}Q_k(s'_k,u)\\
&\qquad\qquad
 -Q_k(s_k,a_k)
\Bigr)
-D(F(Q_k)-Q_k).
\end{aligned}
\end{equation}
Therefore, the vector form of Q-learning is
\[
\begin{gathered}
Q_{k+1}=Q_k+
  \alpha\{D(F(Q_k)-Q_k)+w_k\},\\
k\in\{0,1,2,\ldots\}.
\end{gathered}
\]
where \(w_k\) is defined in~\cref{eq:wk-def}. With
\[
e_k:=Q_k-Q^*,
\]
the Q-learning error recursion is
\[
\begin{gathered}
e_{k+1}=e_k+
  \alpha D\{\gamma P(V_{Q_k}-V^*)-e_k\}
  +\alpha w_k,\\
k\in\{0,1,2,\ldots\}.
\end{gathered}
\]

\begin{assumption}\label{assump:basic}
The following standing conditions hold throughout the paper.
\begin{enumerate}[(i)]
\item \(d(s,a)>0\) for every \((s,a)\in\calS\times\calA\).
\item The step size satisfies \(\alpha\in(0,1)\).
\item The initial Q-table \(Q_0\in\R^{|\calS|\,|\calA|}\) is deterministic.
\end{enumerate}
\end{assumption}

We use the constants
\[
d_{\min}:=\min_{(s,a)\in\calS\times\calA}d(s,a),
  \qquad
  d_{\max}:=\max_{(s,a)\in\calS\times\calA}d(s,a).
\]
By \cref{assump:basic}, \(d_{\min}>0\).

For the stochastic finite-time analysis, define the uniform noise constant
using the reward bound \(R_{\max}\) from \cref{sec:definitions}:
\begin{equation}
\label{eq:wmax-def}
\begin{aligned}
W_{\max}:={}&
  \left(1+\sqrt{|\calS|\,|\calA|}\right)^2\\
&\times
  \left(
  R_{\max}+(1+\gamma)\right.\\
&\qquad\left.
  \max\bigl\{\|Q_0\|_\infty,R_{\max}/(1-\gamma)\bigr\}
  \right)^2.
\end{aligned}
\end{equation}
The conditional moment estimates for \(w_k\), \(w_k^-\), and \(w_k^+\) are
proved in Appendix~\ref{app:noise-proofs} from bounded rewards, the boundedness of the iterates,
and the i.i.d. observation model. The signed increments \(w_k^-\) and \(w_k^+\)
are used only for the pathwise one-step sign comparisons. The stochastic
finite-time estimates below use the martingale difference \(w_k\) through
lower and upper reference filters.

\section{Stochastic Q-Learning and Sign-Separated Analysis}
\label{sec:sign-separated-comparison}
\setcounter{theorem}{9}

This section presents the stochastic switching-system form of Q-learning and
then introduces the sign-separated lower and upper comparison systems.
Since \(V_{Q_k}=\boldsymbol{\Pi}_{Q_k}Q_k\), the stochastic recursion can be
written as
\[
Q_{k+1}
  =
  \mathbf A_{\pi_{Q_k}}Q_k+
  \alpha DR+
  \alpha w_k,
  \qquad k\in\{0,1,2,\ldots\}.
\]
Equivalently, in error coordinates, the error recursion is
\[
\begin{aligned}
e_{k+1}
&=
  \mathbf A_{\pi_{Q_k}}e_k
  -\alpha\gamma DP\bigl(V^*-\boldsymbol{\Pi}^{\pi_{Q_k}}Q^*\bigr)\\
&\quad
  +\alpha w_k,
  \qquad k\in\{0,1,2,\ldots\}.
\end{aligned}
\]
Thus, constant-step-size stochastic Q-learning is an affine stochastic switching
system with an additive noise increment. As in the deterministic
switching-system representation, the affine offset can be removed by
representing the Bellman-max error through a stochastic policy~\citep{lee2026lyapunov}.
The following stochastic linear switching-system representation follows the
direct switching theory of \citet{lee2026lyapunov}.
\begin{lemma}\label{lem:stochastic-exact-direct-switching}
Under the standing Q-learning assumptions, along every trajectory of the
stochastic Q-learning recursion, there exists a sequence of
\(\calF_k\)-measurable stochastic policies \(\{\mu_k\}_{k\ge0}\) such that the stochastic error recursion can be written exactly as the
linear stochastic switching system
\[
\begin{aligned}
e_{k+1}
&=
  \mathbf A_{\mu_k}e_k+\alpha w_k,\\
\mathbf A_{\mu_k}
&:=
  I-\alpha D+\alpha\gamma DP\boldsymbol{\Pi}^{\mu_k},\\
&\hspace{4.6cm} k\in\{0,1,2,\ldots\}.
\end{aligned}
\]
Moreover, each \(\mathbf A_{\mu_k}\) belongs to \(\co(\calM_\alpha)\).
\end{lemma}

\begin{proof}
From the stochastic error recursion,
\[
e_{k+1}
=
  e_k+\alpha D\{\gamma P(V_{Q_k}-V^*)-e_k\}+\alpha w_k.
\]
Fix \(k\) and \(s\). Since \(e_k(s,a)=Q_k(s,a)-Q^*(s,a)\) and
\(A^*(s,a)=V^*(s)-Q^*(s,a)\),
\[
V_{Q_k}(s)-V^*(s)
=
\max_{a\in\calA}\{e_k(s,a)-A^*(s,a)\}.
\]
Because \(A^*(s,a)\ge0\) and at least one action in each state is optimal, this
scalar lies between \(\min_{a\in\calA}e_k(s,a)\) and
\(\max_{a\in\calA}e_k(s,a)\). Hence there exists a stochastic vector
\(\mu_k(s)\in\Delta_{|\calA|}\) such that
\[
V_{Q_k}(s)-V^*(s)
=
\sum_{a\in\calA}\mu_k(s,a)e_k(s,a).
\]
Since the action set is finite and \(e_k\) is \(\calF_k\)-measurable, the vectors
\(\mu_k(s)\) can be chosen \(\calF_k\)-measurably by a fixed tie-breaking rule.
Thus \(V_{Q_k}-V^*=\boldsymbol{\Pi}^{\mu_k}e_k\), and substitution gives
\[
e_{k+1}
=
  (I-\alpha D+\alpha\gamma DP\boldsymbol{\Pi}^{\mu_k})e_k+
  \alpha w_k
=
  \mathbf A_{\mu_k}e_k+
  \alpha w_k.
\]
Finally, every stochastic policy matrix is a convex combination of deterministic
policy matrices: with
\(\lambda_\pi:=\prod_{s\in\calS}\mu_k(s,\pi(s))\), we have
\(\lambda_\pi\ge0\), \(\sum_{\pi\in\Theta}\lambda_\pi=1\), and
\(\boldsymbol{\Pi}^{\mu_k}=\sum_{\pi\in\Theta}\lambda_\pi
\boldsymbol{\Pi}^\pi\). Therefore
\(\mathbf A_{\mu_k}=\sum_{\pi\in\Theta}\lambda_\pi\mathbf A_\pi\in
\co(\calM_\alpha)\).
\end{proof}

This representation and the associated lower/upper comparison systems build on
the switching-system analyses by
\citet{lee2024final,lee2023discrete,lee2020unified,lee2026lyapunov}. The
Bellman-max sandwich and the deterministic residual identities are the same as
in \cref{subsec:deterministic-comparisons}; the only additional term is the
stochastic increment \(\alpha w_k\). Throughout this section, \(Q_k\), \(e_k\),
the greedy selectors introduced below, and all Bellman-max residuals are
\(\calF_k\)-measurable.

\subsection{Summary of Sign-Separated Comparison Systems}
\label{subsec:stochastic-comparison-summary}

The detailed lower, upper, and sign-separated comparison construction is given
in Appendix~\ref{subsec:stochastic-comparison-systems}. We use here only the
consequences needed for the finite-time estimates. With \(\mathbf A_-^\star\)
defined above, let \(\pi_k^+(s)\in\Argmax_{a\in\calA}e_k(s,a)\). The comparison
argument gives the reference filters
\[
\begin{gathered}
\ell_{k+1}=\mathbf A_-^\star\ell_k+\alpha w_k,
  \qquad \ell_0=e_0,\\
u_{k+1}=\mathbf A_{\pi_k^+}u_k+\alpha w_k,
  \qquad u_0=e_0,
\end{gathered}
\]
which satisfy
\[
\ell_k\le e_k\le u_k,
  \qquad k\in\{0,1,2,\ldots\}.
\]
Consequently,
\[
e_k^-\le \ell_k^-,
  \qquad
  e_k^+\le u_k^+,
  \qquad k\in\{0,1,2,\ldots\}.
\]
Thus the negative component is controlled through a fixed optimal-policy LTI
reference filter, whereas the positive component is controlled through a
predictable switching reference filter selected from the full deterministic-policy family.
\subsection{Finite-Time Rates for Sign-Separated Components}\label{sec:finite-time-rates}

\setcounter{theorem}{12}
We next derive finite-time bounds for the two sign components using reference
filters driven by the same martingale difference as the Q-learning error. The
negative component is controlled by the fixed LTI filter generated by
\(\mathbf A_-^\star\). The positive component is controlled by a switching filter
whose matrices range over \(\calM_\alpha\).

\begin{theorem}\label{thm:negative-part-rate}
Assume \cref{assump:basic}, and let \(W_{\max}\) be the noise constant defined
in \cref{eq:wmax-def}. Fix \(\varepsilon>0\) such that
\[
\rho_-^\star+\varepsilon<1,
\]
and define
\[
\beta_-:=\rho_-^\star+\varepsilon,
  \qquad
K_-:=K_{\beta_-}(\{\mathbf A_-^\star\}).
\]
Then, for every \(k\ge0\),
\begin{equation}
\label{eq:negative-part-bound}
\E[\|e_k^-\|_\infty]
  \le
  K_-\beta_-^k\|e_0\|_\infty
  +
  |\calS|\,|\calA|
  \sqrt{\frac{\alpha W_{\max}}{d_{\min}(1-\gamma)}}.
\end{equation}
Consequently, the negative part is certified at the optimized single-policy LTI
rate
\[
\rho_-^\star
  =
  \min_{\pi\in\Theta^*}\rho(\mathbf A_\pi).
\]
\end{theorem}

\begin{proof}
The proof is provided in
Appendix~\ref{app:stochastic-rate-proofs}.
\end{proof}

Since the distance to the nonnegative orthant is exactly the infinity norm of
the negative part, the previous theorem immediately gives the following
orthant-distance estimate.

\begin{corollary}\label{cor:orthant-distance-rate}
Let
\[
\R_+^{|\calS|\,|\calA|}
  :=
  \{x\in\R^{|\calS|\,|\calA|}:x\ge0\}.
\]
Under the assumptions of \cref{thm:negative-part-rate}, for every
\(\varepsilon>0\) satisfying \(\rho_-^\star+\varepsilon<1\), with
\(\beta_-:=\rho_-^\star+\varepsilon\) and
\(K_-:=K_{\beta_-}(\{\mathbf A_-^\star\})\), every \(k\ge0\) satisfies
\[
\begin{aligned}
\E\!\left[
  \dist_\infty(e_k,\R_+^{|\calS|\,|\calA|})
  \right]
&\le
  K_-\beta_-^k\|e_0\|_\infty\\
&\quad+
  |\calS|\,|\calA|
  \sqrt{\frac{\alpha W_{\max}}{d_{\min}(1-\gamma)}}.
\end{aligned}
\]
Consequently, the full error iterate \(e_k\) approaches the nonnegative orthant
at the certified exponential rate
\[
\rho_-^\star
  =
  \min_{\pi\in\Theta^*}\rho(\mathbf A_\pi),
\]
up to the constant-step-size noise floor.
\end{corollary}

\begin{proof}
The proof is provided in
Appendix~\ref{app:stochastic-rate-proofs}.
\end{proof}

The positive component requires one additional reference-filter residual term,
because the switching filter is compared with a fixed reference mode that carries
the same martingale difference.
\begin{theorem}\label{thm:positive-part-rate}
Assume~\cref{assump:basic}, and let \(W_{\max}\) be the noise constant defined
in \cref{eq:wmax-def}. Fix \(\varepsilon>0\) such that
\[
\rho_++\varepsilon<1,
\]
and define
\[
\beta_+:=\rho_++\varepsilon,
  \qquad
K_+:=K_{\beta_+}(\calM_\alpha).
\]
Then, for every \(k\ge0\),
\begin{equation}
\label{eq:positive-part-bound}
\begin{aligned}
\E[\|e_k^+\|_\infty]
&\le
  K_+\beta_+^k\|e_0\|_\infty\\
&\quad+
  2\alpha\gamma d_{\max}K_+^2
  k\beta_+^{k-1}\|e_0\|_\infty\\
&\quad+
  \left(1+\frac{2\gamma d_{\max}}{d_{\min}(1-\gamma)}\right)\\
&\qquad\times
  |\calS|\,|\calA|
  \sqrt{\frac{\alpha W_{\max}}{d_{\min}(1-\gamma)}}.
\end{aligned}
\end{equation}
Here \(k\beta_+^{k-1}\) is interpreted as zero when \(k=0\). This polynomial
prefactor does not change the exponential envelope: for any
\(\tilde\beta\in(\beta_+,1)\), there is a finite constant
\(C_{\tilde\beta}\) such that
\(k\beta_+^{k-1}\le C_{\tilde\beta}\tilde\beta^k\) for all \(k\ge0\).
Consequently, the positive part is certified at the full direct JSR rate
\[
\rho_+=\rho_\alpha^{\mathrm{dir}}
  =
  \rho(\{\mathbf A_\pi:\pi\in\Theta\}).
\]
\end{theorem}

\begin{proof}
The proof is provided in
Appendix~\ref{app:stochastic-rate-proofs}.
\end{proof}

Combining \cref{thm:negative-part-rate,thm:positive-part-rate}, the negative
part is certified at the optimized single-policy LTI rate
\[
\rho_-^\star
  =
  \min_{\pi\in\Theta^*}\rho(\mathbf A_\pi),
\]
whereas the positive part is certified at the full JSR rate
\[
\rho_+=\rho_\alpha^{\mathrm{dir}}
  =
  \rho(\{\mathbf A_\pi:\pi\in\Theta\}).
\]
Since
\[
\rho_-^\star
  \le
  \rho_-
  \le
  \rho_+,
\]
the optimized negative-side LTI certificate is no slower than the
optimal-policy switching certificate and can be strictly faster than the
positive-side direct switching certificate. The positive stochastic estimate has
an additional transient term from the reference-filter residual, but its
certified exponential rate is still \(\rho_+\).

This comparison follows from the sign of the Bellman-max residuals: for every
optimal policy \(\pi\in\Theta^*\),
\[
V_Q-V^*
  \ge
  \boldsymbol{\Pi}^\pi e.
\]
Thus, the residual relative to a fixed optimal policy is nonnegative and does
not enlarge the negative-part upper comparison. Choosing
\[
\pi_-^\star\in\argmin_{\pi\in\Theta^*}\rho(\mathbf A_\pi)
\]
therefore gives the fastest spectral-radius certificate among fixed LTI lower
comparisons associated with optimal policies. In contrast, the positive side
must allow all actions through
\[
V_Q-V^*
  \le
  \max_{a\in\calA}e(s,a).
\]
Thus, suboptimal actions with positive error can enter the positive-side
switching dynamics. This is the structural source of the max-induced
asymmetry.

These conclusions are structural comparison statements, not a universal
positive-bias theorem for Q-learning. They do not by themselves imply
\(\E[e_k]\ge0\) or \(\E[V_{Q_k}-V^*]\ge0\). Rather, they show that negative
errors are dominated by an optimized LTI comparison associated with an optimal
policy, while positive errors may be propagated by the switching family and may
also be sustained by the nonnegative residual
\[
\alpha\gamma DP
  \bigl(V_{Q_k}-V^*-\boldsymbol{\Pi}^{\pi_-^\star}e_k\bigr).
\]
When \(\rho_-^\star<\rho_+\), the negative component has a strictly faster
certified exponential envelope than the positive component, although a
particular sample path or problem instance need not exhibit strictly faster
realized decay of \(e_k^-\). Under a constant step-size, the stochastic bounds
also contain a constant-step-size noise floor; hence the natural interpretation is
entry into a small neighborhood rather than exact convergence to zero in
general.

\setcounter{theorem}{16}
\section{Examples}\label{sec:examples}

The following deterministic example shows that the positive component can realize the slower certificate rate.
\begin{example}\label{ex:slower-positive-deterministic-trajectory}
The following one-state, two-action example shows that the positive part can
have a strictly slower realized decay than the negative part, and also that the
certificate rates can satisfy $\rho_-^\star<\rho_+$. Let $\calS=\{1\}$ and $\calA=\{1,2\}$, and let both actions be self-loops: $P(1\mid 1,1)=P(1\mid 1,2)=1$. Let \(R(1,1)=R(1,2)=0\), \(\gamma=0.9\), and \(\alpha\in(0,1)\), and use the sampling distribution
$d(1,1)=0.9$, $d(1,2)=0.1$. Then, $Q^*(1,1)=Q^*(1,2)=0$. It follows that both actions are optimal and $\Theta^*=\{\pi_1,\pi_2\}$, where $\pi_i(1)=i$ for $i\in\{1,2\}$.
With the ordering \((1,1),(1,2)\), write
\[
e_k=
  \begin{bmatrix}
  Q_k(1,1)\\
  Q_k(1,2)
  \end{bmatrix}.
\]
The selection matrices are $\boldsymbol{\Pi}^{\pi_1}=\begin{bmatrix}1&0\end{bmatrix}$ and $\boldsymbol{\Pi}^{\pi_2}=\begin{bmatrix}0&1\end{bmatrix}$, and
\[
D=
  \begin{bmatrix}
  0.9&0\\
  0&0.1
  \end{bmatrix}.
\]
Therefore the two direct modes are
\[
\begin{gathered}
\mathbf A_{\pi_1}
  =
  \begin{bmatrix}
  1-0.09\alpha&0\\
  0.09\alpha&1-0.1\alpha
  \end{bmatrix},\\
\mathbf A_{\pi_2}
  =
  \begin{bmatrix}
  1-0.9\alpha&0.81\alpha\\
  0&1-0.01\alpha
  \end{bmatrix}.
\end{gathered}
\]
Since these matrices are triangular, their spectral radii are $\rho(\mathbf A_{\pi_1})=1-0.09\alpha$ and $\rho(\mathbf A_{\pi_2})=1-0.01\alpha$.
Therefore, the optimized negative-side LTI certificate is
\[
\rho_-^\star
  =
  \min_{\pi\in\Theta^*}\rho(\mathbf A_\pi)
  =
  1-0.09\alpha.
\]
On the other hand, $\|\mathbf A_{\pi_1}\|_\infty=\|\mathbf A_{\pi_2}\|_\infty=1-0.01\alpha$, hence $\rho(\{\mathbf A_{\pi_1},\mathbf A_{\pi_2}\})\le1-0.01\alpha$.
Since \(\mathbf A_{\pi_2}\) itself has spectral radius \(1-0.01\alpha\), we also have $\rho(\{\mathbf A_{\pi_1},\mathbf A_{\pi_2}\})\ge1-0.01\alpha$.
Consequently,
\[
\rho_+
  =
  \rho(\{\mathbf A_{\pi_1},\mathbf A_{\pi_2}\})
  =
  1-0.01\alpha,
\]
and hence $\rho_-^\star=1-0.09\alpha<1-0.01\alpha=\rho_+$.

Now, choose the initial error vector as
\[
e_0=
  \begin{bmatrix}
  -A\\
  B
  \end{bmatrix},
  \qquad
  A>0,
  \quad
  B>0.
\]
If \(e_k(1,2)>e_k(1,1)\), then the Bellman max selects action \(2\), and the recursion is
\[
e_{k+1}=\mathbf A_{\pi_2}e_k,
  \qquad k\in\{0,1,2,\ldots\}.
\]
Solving this triangular recursion gives $e_k(1,2)=B(1-0.01\alpha)^k$, and
\[
e_k(1,1)
  =
  \frac{81}{89}B(1-0.01\alpha)^k
  -
  \left(A+\frac{81}{89}B\right)(1-0.9\alpha)^k.
\]
Moreover,
\[
\begin{aligned}
e_k(1,2)-e_k(1,1)
&=
  \frac{8}{89}B(1-0.01\alpha)^k\\
&\quad+
  \left(A+\frac{81}{89}B\right)(1-0.9\alpha)^k\\
&>0.
\end{aligned}
\]
for every \(k\ge0\). Hence the branch $e_{k+1}=\mathbf A_{\pi_2}e_k$ is valid for all
\(k\ge0\). Since \(e_k(1,2)>0\) and \(e_k(1,2)>e_k(1,1)\), the positive part satisfies
\[
\|e_k^+\|_\infty
  =
  e_k(1,2)
  =
  B(1-0.01\alpha)^k.
\]
The negative part is
\[
\begin{aligned}
\|e_k^-\|_\infty
&=
  (-e_k(1,1))^+\\
&=
  \Biggl[
  \left(A+\frac{81}{89}B\right)(1-0.9\alpha)^k\\
&\qquad-
  \frac{81}{89}B(1-0.01\alpha)^k
  \Biggr]^+.
\end{aligned}
\]
Hence
\[
\|e_k^-\|_\infty
  \le
  \left(A+\frac{81}{89}B\right)(1-0.9\alpha)^k.
\]
Consequently, on this deterministic trajectory, the positive part decays exactly at the
slow factor $1-0.01\alpha$, whereas the negative part is bounded by the
faster $(1-0.9\alpha)$-geometric term and eventually vanishes.

For the numerical plots in this example, we use the same instance of the
arbitrary initial condition \(e_0=(-A,B)^\top\), namely \(A=20\) and
\(B=1\), for both the deterministic and stochastic simulations. The plotted quantities are the
relative sign errors
\[
\frac{\|e_k^+\|_\infty}{\|e_0^+\|_\infty}
\quad\text{and}\quad
\frac{\|e_k^-\|_\infty}{\|e_0^-\|_\infty}.
\]
\begin{center}
\includegraphics[width=0.98\linewidth]{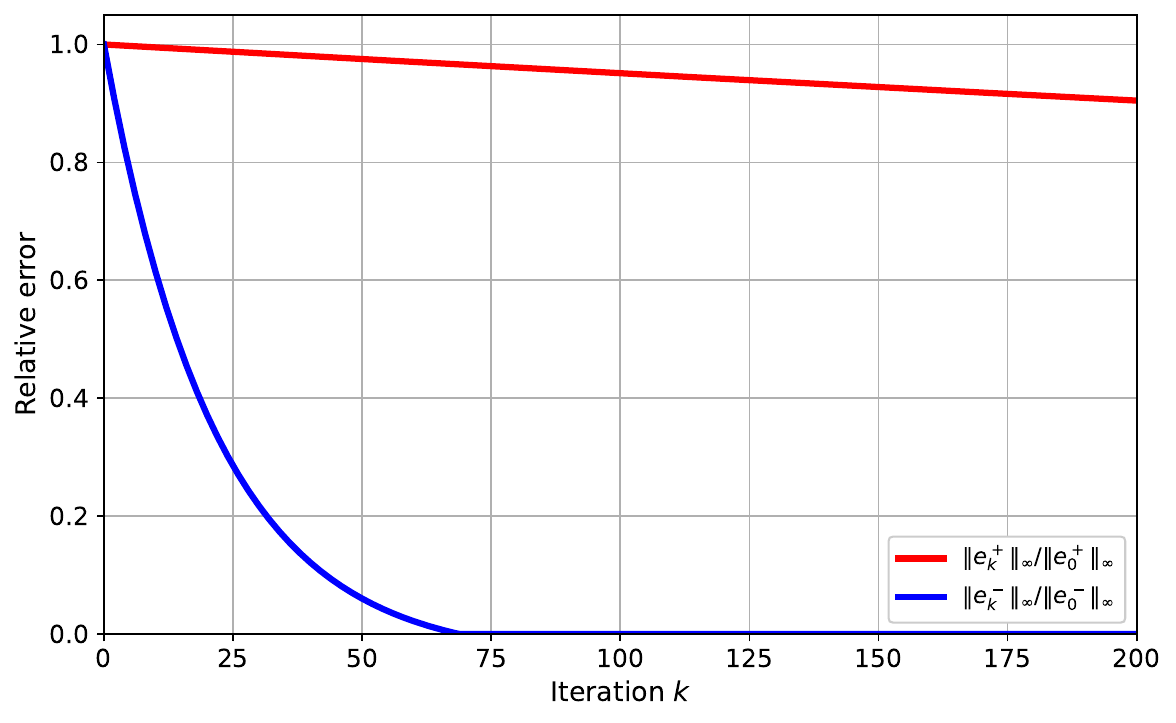}
\captionof{figure}{Relative-error simulation of the deterministic trajectory in
\Cref{ex:slower-positive-deterministic-trajectory} with \(\alpha=0.05\) and the shared
initial condition \(e_0=(-20,1)^\top\). The red curve shows
\(\|e_k^+\|_\infty/\|e_0^+\|_\infty\), and the blue curve shows
\(\|e_k^-\|_\infty/\|e_0^-\|_\infty\).}
\label{fig:ex1-positive-negative}
\end{center}
The simulation in \Cref{fig:ex1-positive-negative} illustrates the preceding formulas:
the normalized positive part follows the slow factor $1-0.01\alpha$, while the
normalized negative part decreases much faster and eventually becomes zero.

We also consider the corresponding stochastic asynchronous Q-learning update for the
same MDP. Starting from the same initial condition \(e_0=(-20,1)^\top\), we simulate one independent run of this stochastic recursion with \(\alpha=0.05\).
\begin{center}
\includegraphics[width=0.98\linewidth]{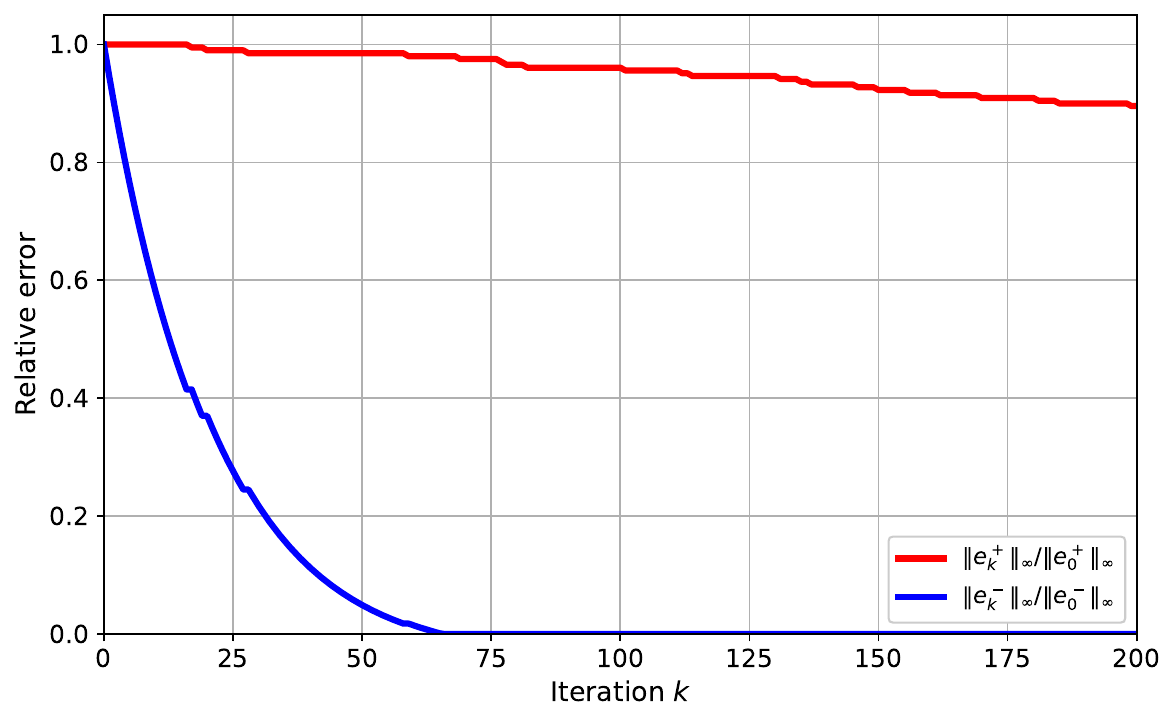}
\captionof{figure}{Relative-error simulation of one stochastic asynchronous Q-learning run
for the same MDP with \(\alpha=0.05\) and the shared initial condition
\(e_0=(-20,1)^\top\). The red curve shows
\(\|e_k^+\|_\infty/\|e_0^+\|_\infty\), and the blue curve shows
\(\|e_k^-\|_\infty/\|e_0^-\|_\infty\).}
\label{fig:ex1-stochastic-positive-negative}
\end{center}
The stochastic simulation in \Cref{fig:ex1-stochastic-positive-negative} shows the same qualitative separation after normalization: the positive component follows a slow decay profile, whereas the negative component is removed much more rapidly.
\end{example}

\section{Conclusion}\label{sec:conclusion}

This paper developed a sign-separated finite-time analysis of constant-step-size
Q-learning from a switching-system viewpoint. The negative component is bounded
by an optimized fixed-mode linear system associated with an optimal policy,
whereas the positive component requires a switching-family bound over all
deterministic policies. The stochastic bounds retain this asymmetry up to the
constant-step-size noise floor, while the deterministic counterpart is given in
Appendix~\ref{sec:deterministic-qlearning-analysis}. These results are
comparison certificates rather than universal sign-bias claims, and they
complement contraction-based and stochastic-approximation analyses of
Q-learning.


\clearpage
\appendix

\setcounter{theorem}{2}
\section{Deterministic Q-Learning}
\label{sec:deterministic-qlearning-analysis}

Before the stochastic Q-learning analysis, we study the corresponding
noise-free deterministic analysis. The deterministic recursion is the
conditional-mean counterpart of the asynchronous stochastic Q-learning update:
\[
Q_{k+1}
  =
  Q_k+
  \alpha D(F(Q_k)-Q_k),
  \qquad k\in\{0,1,2,\ldots\}.
\]
With $e_k:=Q_k-Q^*$, we have
\begin{equation}
\label{eq:deterministic-error-recursion}
\begin{gathered}
e_{k+1}=e_k+
  \alpha D\{\gamma P(V_{Q_k}-V^*)-e_k\},\\
k\in\{0,1,2,\ldots\}.
\end{gathered}
\end{equation}
Since \(V_{Q_k}=\boldsymbol{\Pi}_{Q_k}Q_k\), the deterministic recursion also has the
affine switching-system representation~\citep{lee2020unified,lee2023discrete,
lee2024final}
\[
Q_{k+1}
  =
  \mathbf A_{\pi_{Q_k}}Q_k+
  \alpha DR,
  \qquad k\in\{0,1,2,\ldots\},
\]
where, for each deterministic policy \(\pi\in\Theta\),
\[
\mathbf A_\pi
  :=
  I-\alpha D+\alpha\gamma DP\boldsymbol{\Pi}^\pi.
\]
Equivalently, in error coordinates, the error recursion can be written as
\[
\begin{gathered}
e_{k+1}
  =
  \mathbf A_{\pi_{Q_k}}e_k
  -\alpha\gamma DP\bigl(V^*-\boldsymbol{\Pi}^{\pi_{Q_k}}Q^*\bigr),\\
k\in\{0,1,2,\ldots\}.
\end{gathered}
\]
Thus, deterministic Q-learning is an affine discrete-time switching system whose
mode is selected by the greedy policy induced by the current iterate~\citep{lee2024final,lee2023discrete,lee2020unified}.
The affine offset can also be removed exactly by the stochastic-policy
linearization of the Bellman maximization error as in \citet{lee2026lyapunov}.
\begin{lemma}\label{lem:deterministic-exact-direct-switching}
Along every trajectory of the deterministic recursion, there exists a sequence
of stochastic policies \(\{\mu_k\}_{k\ge0}\) such that the deterministic error recursion can be written exactly as the linear switching system
\[
\begin{aligned}
e_{k+1}&=
  \mathbf A_{\mu_k}e_k,\\
\mathbf A_{\mu_k}
&:=
  I-\alpha D+\alpha\gamma DP\boldsymbol{\Pi}^{\mu_k},\\
&\hspace{4.6cm} k\in\{0,1,2,\ldots\}.
\end{aligned}
\]
Moreover, each \(\mathbf A_{\mu_k}\) belongs to \(\co(\calM_\alpha)\).
\end{lemma}

\begin{proof}
The same convex-combination argument as in
\cref{lem:stochastic-exact-direct-switching} applies with \(w_k=0\). For each
state \(s\), represent \(V_{Q_k}(s)-V^*(s)\) as a convex combination of the
finite set \(\{e_k(s,a):a\in\calA\}\), form the stochastic policy \(\mu_k\), and
substitute \(V_{Q_k}-V^*=\boldsymbol{\Pi}^{\mu_k}e_k\) into the deterministic
error recursion. The inclusion \(\mathbf A_{\mu_k}\in\co(\calM_\alpha)\) follows
from the product-weight expansion of a stochastic policy as a convex combination
of deterministic policies.
\end{proof}

The corresponding switching family is
\[
\calM_\alpha
  :=
  \{\mathbf A_\pi:\pi\in\Theta\}.
\]
The corresponding JSR is
\[
\rho_\alpha^{\mathrm{dir}}  :=   \rho(\calM_\alpha).
\]
Although the exact direct representation may use matrices in
\(\co(\calM_\alpha)\), the same JSR certificate is obtained from the
deterministic-policy family because
\[
\rho(\co(\calM_\alpha))=\rho(\calM_\alpha).
\]
Indeed, \(\calM_\alpha\subseteq\co(\calM_\alpha)\) gives one inequality,
while every product of matrices from \(\co(\calM_\alpha)\) expands as a convex
combination of products of matrices from \(\calM_\alpha\), giving the reverse
inequality after taking norms and the joint spectral-radius limit.

\subsection{Lower and Upper Comparison Systems}
\label{subsec:deterministic-comparisons}

We first recall the lower and upper comparison-system framework developed by
\citet{lee2024final,lee2023discrete,lee2020unified}. The detailed Bellman-max
expansion and residual-sign derivation are given in
Appendix~\ref{app:deterministic-comparison-proofs}.
We state only the comparison systems and the resulting order statement.

For the lower comparison, choose a single optimal policy whose associated LTI
subsystem has the smallest spectral radius:
\[
\pi_-^\star
  \in
  \argmin_{\pi\in\Theta^*}\rho(\mathbf A_\pi).
\]
The minimizer exists because \(\Theta^*\) is finite. Define
\[
\mathbf A_-^\star:=\mathbf A_{\pi_-^\star},
  \qquad
  \rho_-^\star
  :=
  \rho(\mathbf A_-^\star)
  =
  \min_{\pi\in\Theta^*}\rho(\mathbf A_\pi).
\]
Then, the lower comparison system is
\begin{equation}
\label{eq:deterministic-lower-comparison-system}
\ell_{k+1}
  =
  \mathbf A_-^\star\ell_k,
  \qquad
  \ell_0=e_0,
  \qquad k\in\{0,1,2,\ldots\}.
\end{equation}
It corresponds to dropping a nonnegative Bellman-max residual from the exact
error identity relative to the fixed optimal policy \(\pi_-^\star\).

For the upper comparison system, let us define the state-wise maximizer of the current error
\[
\pi_k^+(s)\in\Argmax_{a\in\calA}e_k(s,a),
  \qquad s\in\calS.
\]
Then, the upper comparison system is
\begin{equation}
\label{eq:deterministic-upper-comparison-system}
u_{k+1}
  =
  \mathbf A_{\pi_k^+}u_k,
  \qquad
  u_0=e_0,
  \qquad k\in\{0,1,2,\ldots\}.
\end{equation}
It corresponds to dropping a nonnegative residual in the opposite direction,
where the active mode is selected by the largest state-wise error component.
Thus,~\cref{eq:deterministic-lower-comparison-system} is a fixed-mode LTI
comparison, whereas \cref{eq:deterministic-upper-comparison-system} is a linear
switching comparison driven by \(\pi_k^+\).

\begin{lemma}\label{lem:deterministic-lower-upper-order}
Under~\cref{assump:basic}, with \(\ell_k\) and \(u_k\) defined in~\cref{eq:deterministic-lower-comparison-system,eq:deterministic-upper-comparison-system}, we have
\[
\ell_k\le e_k\le u_k,
  \qquad k\in\{0,1,2,\ldots\}.
\]
\end{lemma}

\begin{proof}
We prove the two inequalities directly from the residual identities. First, set
\[
\delta_k:=e_k-\ell_k.
\]
Using \cref{lem:deterministic-lower-error-identity} and the lower comparison
recursion in \cref{eq:deterministic-lower-comparison-system}, we obtain
\[
\begin{aligned}
  \delta_{k+1}
  &=e_{k+1}-\ell_{k+1} \\
  &=\left[
  \mathbf A_-^\star e_k+
  \alpha\gamma DP
  \bigl(V_{Q_k}-V^*-\boldsymbol{\Pi}^{\pi_-^\star}e_k\bigr)
  \right]
  -\mathbf A_-^\star\ell_k \\
  &=\mathbf A_-^\star(e_k-\ell_k)+
  \alpha\gamma DP
  \bigl(V_{Q_k}-V^*-\boldsymbol{\Pi}^{\pi_-^\star}e_k\bigr) \\
  &=\mathbf A_-^\star\delta_k+
  \alpha\gamma DP
  \bigl(V_{Q_k}-V^*-\boldsymbol{\Pi}^{\pi_-^\star}e_k\bigr).
\end{aligned}
\]
Since \(\delta_0=e_0-\ell_0=0\), \(\mathbf A_-^\star\ge0\),
\(D\ge0\), \(P\ge0\), and
\[
V_{Q_k}-V^*-\boldsymbol{\Pi}^{\pi_-^\star}e_k\ge0,
\]
induction gives \(\delta_k\ge0\), that is,
\[
\ell_k\le e_k,
  \qquad k\in\{0,1,2,\ldots\}.
\]

Second, set
\[
q_k:=u_k-e_k.
\]
Using \cref{lem:deterministic-upper-error-identity} and the upper comparison
recursion in \cref{eq:deterministic-upper-comparison-system}, we obtain
\[
\begin{aligned}
  q_{k+1}
  &=u_{k+1}-e_{k+1} \\
  &=\mathbf A_{\pi_k^+}u_k-
  \mathbf A_{\pi_k^+}e_k\\
  &\quad+
  \alpha\gamma DP
  \bigl(\boldsymbol{\Pi}^{\pi_k^+}e_k-(V_{Q_k}-V^*)\bigr) \\
  &=\mathbf A_{\pi_k^+}(u_k-e_k)\\
  &\quad+
  \alpha\gamma DP
  \bigl(\boldsymbol{\Pi}^{\pi_k^+}e_k-(V_{Q_k}-V^*)\bigr) \\
  &=\mathbf A_{\pi_k^+}q_k\\
  &\quad+
  \alpha\gamma DP
  \bigl(\boldsymbol{\Pi}^{\pi_k^+}e_k-(V_{Q_k}-V^*)\bigr).
\end{aligned}
\]
Since \(q_0=u_0-e_0=0\), \(\mathbf A_{\pi_k^+}\ge0\),
\(D\ge0\), \(P\ge0\), and
\[
\boldsymbol{\Pi}^{\pi_k^+}e_k-(V_{Q_k}-V^*)\ge0,
\]
induction gives \(q_k\ge0\), that is,
\[
e_k\le u_k,
  \qquad k\in\{0,1,2,\ldots\}.
\]
Combining the two inequalities yields
\[
\ell_k\le e_k\le u_k,
  \qquad k\in\{0,1,2,\ldots\}.
\]
\end{proof}

\subsection{Sign Comparison Systems}
\label{subsec:deterministic-sign-comparisons}

The sign comparison systems below follow from the same residual signs. Write
\(e_k=e_k^+-e_k^-\). The lower residual identity gives
\[
e_{k+1}
  =
  \mathbf A_-^\star e_k^+
  -
  \mathbf A_-^\star e_k^-
  +
  \alpha\gamma DP
  \bigl(V_{Q_k}-V^*-\boldsymbol{\Pi}^{\pi_-^\star}e_k\bigr),
\]
where the last term is nonnegative. Indeed, \(\mathbf A_-^\star e_k^+\)
and the residual term are nonnegative, and \(\mathbf A_-^\star e_k^-\ge0\).
It follows that each coordinate of \(e_{k+1}\) has the form \(a_i-b_i\), with
\(a_i\ge0\) and \(b_i=(\mathbf A_-^\star e_k^-)_i\ge0\). Since
\((a_i-b_i)^-\le b_i\), we obtain
\[
e_{k+1}^-
  \le
  \mathbf A_-^\star e_k^-.
\]
Similarly, the upper residual identity gives
\[
e_{k+1}
  =
  \mathbf A_{\pi_k^+}e_k^+
  -
  \mathbf A_{\pi_k^+}e_k^-
  -
  \alpha\gamma DP
  \bigl(\boldsymbol{\Pi}^{\pi_k^+}e_k-(V_{Q_k}-V^*)\bigr),
\]
where the subtracted residual is nonnegative. Therefore, we similarly obtain
\[
e_{k+1}^+
  \le
  \mathbf A_{\pi_k^+}e_k^+.
\]
These inequalities lead to the following comparison systems for
the negative and positive parts. The detailed one-step proof is included in~Appendix~\ref{app:deterministic-comparison-proofs}.

\begin{lemma}\label{lem:deterministic-pathwise-sign-comparison}
Under \cref{assump:basic}, define
\begin{equation}
\label{eq:deterministic-negative-sign-comparison-system}
z^-_{k+1}
  =
  \mathbf A_-^\star z^-_k,
  \qquad
  z^-_0=e_0^-,
  \qquad k\in\{0,1,2,\ldots\},
\end{equation}
and
\begin{equation}
\label{eq:deterministic-positive-sign-comparison-system}
z^+_{k+1}
  =
  \mathbf A_{\pi_k^+}z^+_k,
  \qquad
  z^+_0=e_0^+,
  \qquad k\in\{0,1,2,\ldots\}.
\end{equation}
Then
\[
e_k^-\le z^-_k,
  \qquad
  e_k^+\le z^+_k,
  \qquad k\in\{0,1,2,\ldots\}.
\]
\end{lemma}

\begin{proof}
At \(k=0\), the inequalities hold by definition. If they hold at time \(k\), then the one-step sign inequalities stated above and the nonnegativity of
\(\mathbf A_-^\star\) and \(\mathbf A_{\pi_k^+}\) give the inequalities at time
\(k+1\). The result follows by induction.
\end{proof}

The formal deterministic comparison lemmas and their proofs are given in
Appendix~\ref{app:deterministic-comparison-proofs}. These sign-separated recursions are the finite-time objects used below:
the negative part evolves under fixed-mode LTI dynamics, whereas the positive
part is propagated by switching-system dynamics.

The positive-side comparison uses the full switching family
\[
\calM_\alpha^+
  :=
  \{\mathbf A_\pi:\pi\in\Theta\}
  =
  \calM_\alpha,
\]
with
\[
\rho_+
  :=
  \rho(\calM_\alpha^+)
  =
  \rho_\alpha^{\mathrm{dir}}.
\]
For comparison with the optimized negative-side LTI certificate, define the
optimal-policy switching family
\[
\calM_\alpha^{-}
  :=
  \{\mathbf A_\pi:\pi\in\Theta^*\},
\]
with JSR
\[
\rho_-
  :=
  \rho(\calM_\alpha^{-}).
\]
Since \(\Theta^*\subseteq\Theta\), we have
\[
\rho_-^\star
  \le
  \rho_-
  \le
  \rho_+.
\]
Thus, the optimized negative-side LTI certificate is no slower than the
optimal-policy switching certificate at the level of spectral-radius
certificates, and it can be strictly faster than the positive-side direct
switching certificate.

\begin{lemma}\label{lem:rho-plus-rho-minus-strictly-stable}
Under \cref{assump:basic}, the direct rates satisfy
\[
\begin{gathered}
\rho_+
  \le
  1-\alpha(1-\gamma)d_{\min}
  <1,\\
\rho_-^\star
  \le
  1-\alpha(1-\gamma)d_{\min}
  <1.
\end{gathered}
\]
Consequently, there exist \(\varepsilon>0\) such that
\(\rho_++\varepsilon<1\) and \(\rho_-^\star+\varepsilon<1\).
\end{lemma}

\begin{proof}
For any \(\pi\in\Theta\), the matrix
\(\mathbf A_\pi=I-\alpha D+\alpha\gamma DP\boldsymbol{\Pi}^\pi\) is
nonnegative. Moreover, \(P\boldsymbol{\Pi}^\pi\) is row stochastic, so the row
sum of \(\mathbf A_\pi\) at coordinate \((s,a)\) is
\[
\begin{aligned}
&1-\alpha d(s,a)+\alpha\gamma d(s,a)\\
&\quad=
  1-\alpha(1-\gamma)d(s,a)\\
&\quad\le
  1-\alpha(1-\gamma)d_{\min}.
\end{aligned}
\]
Thus
\[
\|\mathbf A_\pi\|_\infty
  \le
  1-\alpha(1-\gamma)d_{\min}
  <1
\]
for every \(\pi\in\Theta\). Taking products and then the joint
spectral-radius limit yields the bound on \(\rho_+\). Since
\(\rho_-^\star\le\rho_+\), the same strict upper bound holds for
\(\rho_-^\star\).
\end{proof}

\subsection{Deterministic Finite-Time Rates}
\label{subsec:deterministic-finite-time-rates}

We now convert the sign comparison systems into deterministic finite-time
bounds. The argument uses only product growth estimates from the JSR.

\begin{theorem}\label{thm:deterministic-negative-rate}
Assume \cref{assump:basic}, and fix \(\varepsilon>0\) such that
\[
\rho_-^\star+\varepsilon<1.
\]
Set
\[
\beta_-:=\rho_-^\star+\varepsilon,
  \qquad
K_-:=K_{\beta_-}(\{\mathbf A_-^\star\}).
\]
Then every deterministic Q-learning trajectory satisfies
\[
\|e_k^-\|_\infty
  \le
  K_-\beta_-^k\|e_0^-\|_\infty,
  \qquad k\ge0.
\]
Consequently, the deterministic negative part is certified at rate
\(\rho_-^\star\).
\end{theorem}

\begin{proof}
By \cref{lem:deterministic-pathwise-sign-comparison},
\(e_k^-\le z_k^-\), where
\[
z_k^-=(\mathbf A_-^\star)^k e_0^-.
\]
Since \(\beta_->\rho(\mathbf A_-^\star)\), \cref{lem:product-growth-bound}
applied to the singleton family \(\{\mathbf A_-^\star\}\) gives
\[
\|(\mathbf A_-^\star)^k\|_\infty\le K_-\beta_-^k.
\]
The claimed estimate follows from nonnegativity and the infinity norm.
\end{proof}

\begin{theorem}\label{thm:deterministic-positive-rate}
Assume \cref{assump:basic}, and fix \(\varepsilon>0\) such that
\[
\rho_++\varepsilon<1.
\]
Set
\[
\beta_+:=\rho_++\varepsilon,
  \qquad
K_+:=K_{\beta_+}(\calM_\alpha).
\]
Then every deterministic Q-learning trajectory satisfies
\[
\|e_k^+\|_\infty
  \le
  K_+\beta_+^k\|e_0^+\|_\infty,
  \qquad k\ge0.
\]
Consequently, the deterministic positive part is certified at rate \(\rho_+\).
\end{theorem}

\begin{proof}
By \cref{lem:deterministic-pathwise-sign-comparison},
\(e_k^+\le z_k^+\), where
\[
z_k^+
  =
  \mathbf A_{\pi_{k-1}^+}\cdots\mathbf A_{\pi_0^+}e_0^+
\]
for \(k\ge1\), and the product is the identity for \(k=0\). Since each
\(\mathbf A_{\pi_i^+}\) belongs to \(\calM_\alpha\),
\cref{lem:product-growth-bound} gives
\[
\|\mathbf A_{\pi_{k-1}^+}\cdots\mathbf A_{\pi_0^+}\|_\infty
  \le
  K_+\beta_+^k.
\]
The desired bound follows.
\end{proof}

\begin{corollary}\label{cor:deterministic-orthant-distance-rate}
Under the assumptions of \cref{thm:deterministic-negative-rate},
\[
\dist_\infty(e_k,\R_+^{|\calS|\,|\calA|})
  \le
  K_-\beta_-^k\|e_0^-\|_\infty,
  \qquad k\ge0.
\]
\end{corollary}

\begin{proof}
For every vector \(x\),
\(\dist_\infty(x,\R_+^{|\calS|\,|\calA|})=\|x^-\|_\infty\). Applying this
identity with \(x=e_k\) and using \cref{thm:deterministic-negative-rate}
proves the claim.
\end{proof}

The deterministic estimates have the same rate separation as the spectral-radius
certificates:
\[
\rho_-^\star\le \rho_-\le \rho_+.
\]
Thus the negative component is controlled by the optimized fixed optimal-policy
mode, whereas the positive component must allow the full switching family.

\setcounter{theorem}{20}
\subsection{Deterministic comparison lemmas}\label{app:deterministic-comparison-proofs}

For any \(Q\) and \(e=Q-Q^*\), the Bellman-max term can be expanded state by
state as
\[
\begin{aligned}
  V_Q(s)-V^*(s)
  &= \max_{a\in\calA}\{Q^*(s,a)+e(s,a)\}-V^*(s) \\
  &= \max_{a\in\calA}\{e(s,a)-A^*(s,a)\}.
\end{aligned}
\]
If \(\pi\in\Theta^*\), then \(A^*(s,\pi(s))=0\), and therefore
\[
V_Q-V^*-\boldsymbol{\Pi}^{\pi}e\ge0.
\]
For the upper comparison, if
\(\pi^+(s)\in\Argmax_{a\in\calA}e(s,a)\), then \(A^*(s,a)\ge0\) gives
\[
\boldsymbol{\Pi}^{\pi^+}e-(V_Q-V^*)\ge0.
\]
Substituting these two decompositions into the deterministic error recursion
produces the lower and upper residual identities below.

The first deterministic comparison lemma states the lower residual identity and its sign.

\begin{lemma}\label{lem:deterministic-lower-error-identity}
Under \cref{assump:basic}, the deterministic error recursion satisfies
\[
\begin{gathered}
e_{k+1}
  =
  \mathbf A_-^\star e_k
  +\alpha\gamma DP
  \bigl(V_{Q_k}-V^*-\boldsymbol{\Pi}^{\pi_-^\star}e_k\bigr),\\
k\in\{0,1,2,\ldots\},
\end{gathered}
\]
and
\[
V_{Q_k}-V^*-\boldsymbol{\Pi}^{\pi_-^\star}e_k\ge0.
\]
\end{lemma}

\begin{proof}[Proof of \cref{lem:deterministic-lower-error-identity}]
For each state \(s\), because \(\pi_-^\star\in\Theta^*\), we have
\(Q^*(s,\pi_-^\star(s))=V^*(s)\). Hence
\[
\begin{aligned}
& V_{Q_k}(s)-V^*(s)-e_k(s,\pi_-^\star(s)) \\
&\quad=
  \max_{a\in\calA}\{Q^*(s,a)+e_k(s,a)\}\\
&\qquad -V^*(s)-e_k(s,\pi_-^\star(s)) \\
&\quad\ge
  Q^*(s,\pi_-^\star(s))+e_k(s,\pi_-^\star(s))\\
&\qquad -V^*(s)-e_k(s,\pi_-^\star(s)) \\
&\quad=0.
\end{aligned}
\]
Using
\[
\begin{aligned}
V_{Q_k}-V^*
&=
  \boldsymbol{\Pi}^{\pi_-^\star}e_k\\
&\quad+
  \bigl(V_{Q_k}-V^*-\boldsymbol{\Pi}^{\pi_-^\star}e_k\bigr).
\end{aligned}
\]
in \cref{eq:deterministic-error-recursion}, we obtain
\[
\begin{aligned}
  e_{k+1}
  &=e_k+
  \alpha D\{\gamma P\boldsymbol{\Pi}^{\pi_-^\star}e_k-e_k\} \\
  &\quad+
  \alpha\gamma DP
  \bigl(V_{Q_k}-V^*-\boldsymbol{\Pi}^{\pi_-^\star}e_k\bigr) \\
  &=
  \bigl(I-\alpha D+\alpha\gamma DP\boldsymbol{\Pi}^{\pi_-^\star}\bigr)e_k\\
  &\quad+
  \alpha\gamma DP
  \bigl(V_{Q_k}-V^*-\boldsymbol{\Pi}^{\pi_-^\star}e_k\bigr) \\
  &=
  \mathbf A_-^\star e_k+
  \alpha\gamma DP
  \bigl(V_{Q_k}-V^*-\boldsymbol{\Pi}^{\pi_-^\star}e_k\bigr).
\end{aligned}
\]
This proves the lower residual identity.
\end{proof}

The next lemma states the corresponding upper residual identity.

\begin{lemma}\label{lem:deterministic-upper-error-identity}
Under \cref{assump:basic}, the deterministic error recursion satisfies
\[
\begin{gathered}
e_{k+1}
  =
  \mathbf A_{\pi_k^+}e_k
  -\alpha\gamma DP
  \bigl(\boldsymbol{\Pi}^{\pi_k^+}e_k-(V_{Q_k}-V^*)\bigr),\\
k\in\{0,1,2,\ldots\},
\end{gathered}
\]
where
\[
\pi_k^+(s)\in\Argmax_{a\in\calA}e_k(s,a),
  \qquad s\in\calS,
\]
and
\[
\boldsymbol{\Pi}^{\pi_k^+}e_k-(V_{Q_k}-V^*)\ge0.
\]
\end{lemma}

\begin{proof}[Proof of \cref{lem:deterministic-upper-error-identity}]
Since \(A^*(s,a)=V^*(s)-Q^*(s,a)\ge0\), we have
\[
\begin{aligned}
V_{Q_k}(s)-V^*(s)
&=
  \max_{a\in\calA}\{e_k(s,a)-A^*(s,a)\}\\
&\le
  \max_{a\in\calA}e_k(s,a)\\
&=
  e_k(s,\pi_k^+(s)).
\end{aligned}
\]
Therefore, \(\boldsymbol{\Pi}^{\pi_k^+}e_k-(V_{Q_k}-V^*)\ge0\). Using
\[
\begin{aligned}
V_{Q_k}-V^*
&=
  \boldsymbol{\Pi}^{\pi_k^+}e_k\\
&\quad-
  \bigl(\boldsymbol{\Pi}^{\pi_k^+}e_k-(V_{Q_k}-V^*)\bigr).
\end{aligned}
\]
in \cref{eq:deterministic-error-recursion}, we obtain
\[
\begin{aligned}
  e_{k+1}
  &=e_k+
  \alpha D\{\gamma P\boldsymbol{\Pi}^{\pi_k^+}e_k-e_k\} \\
  &\quad-
  \alpha\gamma DP
  \bigl(\boldsymbol{\Pi}^{\pi_k^+}e_k-(V_{Q_k}-V^*)\bigr) \\
  &=
  \bigl(I-\alpha D+\alpha\gamma DP\boldsymbol{\Pi}^{\pi_k^+}\bigr)e_k\\
  &\quad-
  \alpha\gamma DP
  \bigl(\boldsymbol{\Pi}^{\pi_k^+}e_k-(V_{Q_k}-V^*)\bigr) \\
  &=
  \mathbf A_{\pi_k^+}e_k-
  \alpha\gamma DP
  \bigl(\boldsymbol{\Pi}^{\pi_k^+}e_k-(V_{Q_k}-V^*)\bigr).
\end{aligned}
\]
This proves the upper residual identity.
\end{proof}

The following lemma connects these order comparisons to the componentwise sign parts.

\begin{lemma}\label{lem:deterministic-sign-controlled-by-comparisons}
Under \cref{assump:basic}, with \(\ell_k\) and \(u_k\) defined in
\cref{eq:deterministic-lower-comparison-system,eq:deterministic-upper-comparison-system},
\[
e_k^-\le \ell_k^-,
  \qquad
  e_k^+\le u_k^+,
  \qquad k\in\{0,1,2,\ldots\}.
\]
\end{lemma}

\begin{proof}[Proof of \cref{lem:deterministic-sign-controlled-by-comparisons}]
Fix a coordinate \(i\) and a time \(k\). From
\cref{lem:deterministic-lower-upper-order},
\[
\ell_{k,i}\le e_{k,i}\le u_{k,i}.
\]
The positive-part map \(x\mapsto x^+=\max\{x,0\}\) is monotone increasing on
\(\R\). Therefore, the upper comparison inequality gives
\[
e_{k,i}^+=(e_{k,i})^+
  \le
  (u_{k,i})^+=u_{k,i}^+.
\]
For the negative part, the lower comparison inequality \(\ell_{k,i}\le e_{k,i}\)
implies
\[
-e_{k,i}\le -\ell_{k,i}.
\]
Applying the same monotonicity to the positive-part map gives
\[
e_{k,i}^-=(-e_{k,i})^+
  \le
  (-\ell_{k,i})^+=\ell_{k,i}^-.
\]
Since the coordinate \(i\) was arbitrary, these scalar inequalities hold
componentwise, and hence
\[
e_k^-\le \ell_k^-,
  \qquad
  e_k^+\le u_k^+,
  \qquad k\in\{0,1,2,\ldots\}.
\]
\end{proof}

The next lemma gives the one-step sign domination that drives the sign-separated systems.

\begin{lemma}\label{lem:deterministic-one-step-domination}
Under \cref{assump:basic}, the deterministic negative part satisfies
\[
e_{k+1}^-
  \le
  \mathbf A_-^\star e_k^-,
  \qquad k\in\{0,1,2,\ldots\},
\]
and the deterministic positive part satisfies
\[
e_{k+1}^+
  \le
  \mathbf A_{\pi_k^+}e_k^+,
  \qquad k\in\{0,1,2,\ldots\}.
\]
\end{lemma}

\begin{proof}[Proof of \cref{lem:deterministic-one-step-domination}]
From \cref{lem:deterministic-lower-error-identity} and
\(e_k=e_k^+-e_k^-\),
\[
\begin{aligned}
e_{k+1}
&=
  \mathbf A_-^\star e_k^+
  -\mathbf A_-^\star e_k^-+r_k^-,\\
r_k^-&:=\alpha\gamma DP
  \bigl(V_{Q_k}-V^*-\boldsymbol{\Pi}^{\pi_-^\star}e_k\bigr)\ge0.
\end{aligned}
\]
Because \(\mathbf A_-^\star\ge0\) and \(e_k^+,e_k^-\ge0\), the vectors
\[
a_k^-:=\mathbf A_-^\star e_k^+ + r_k^-,
  \qquad
  b_k^-:=\mathbf A_-^\star e_k^-
\]
are nonnegative and satisfy \(e_{k+1}=a_k^- - b_k^-\). Hence, component by
component,
\[
(e_{k+1,i})^-
  =
  (b_{k,i}^- - a_{k,i}^-)^+
  \le
  b_{k,i}^-.
\]
Thus,
\[
e_{k+1}^-
  \le
  b_k^-
  =
  \mathbf A_-^\star e_k^-.
\]
Similarly, from \cref{lem:deterministic-upper-error-identity},
\[
\begin{aligned}
e_{k+1}
&=
  \mathbf A_{\pi_k^+}e_k^+
  -\mathbf A_{\pi_k^+}e_k^--r_k^+,\\
r_k^+&:=\alpha\gamma DP
  \bigl(\boldsymbol{\Pi}^{\pi_k^+}e_k-(V_{Q_k}-V^*)\bigr)\ge0.
\end{aligned}
\]
Set
\[
a_k^+:=\mathbf A_{\pi_k^+}e_k^+,
  \qquad
  b_k^+:=\mathbf A_{\pi_k^+}e_k^-+r_k^+.
\]
Then \(a_k^+,b_k^+\ge0\) and \(e_{k+1}=a_k^+ - b_k^+\). Hence
\[
(e_{k+1,i})^+
  =
  (a_{k,i}^+ - b_{k,i}^+)^+
  \le
  a_{k,i}^+,
\]
which yields
\[
e_{k+1}^+
  \le
  a_k^+
  =
  \mathbf A_{\pi_k^+}e_k^+.
\]
\end{proof}

\setcounter{theorem}{17}
\subsection{Additional example}\label{app:examples}

The following example clarifies that the certificate asymmetry need not force different realized decay rates on every trajectory.

\begin{example}\label{ex:equal-realized-positive-negative-errors}
The comparison rates above are certificate rates, and they need not imply that
one sign component has a strictly faster realized decay on every deterministic
trajectory. The following two-state, two-action example gives a simple
trajectory on which the positive and negative errors have exactly the same
magnitude and the same decay factor.

Let $\calS=\{1,2\}$ and $\calA=\{1,2\}$. For every state and action, let the transition be a self-loop,
\[
P(s\mid s,i)=1,
  \qquad
  s\in\{1,2\},\ i\in\{1,2\},
\]
and let the rewards be
\[
R(s,1)=0,
  \qquad
  R(s,2)=-1,
  \qquad
  s\in\{1,2\}.
\]
Then
\[
V^*(s)=0,
  \qquad
  Q^*(s,1)=0,
  \qquad
  Q^*(s,2)=-1,
\]
Thus, action \(1\) is the unique optimal action in each state and the action gap
of action \(2\) is equal to one. Consider the deterministic recursion
\cref{eq:deterministic-error-recursion} with uniform sampling
\[
d(s,i)=\frac{1}{4},
  \qquad
  s\in\{1,2\},\ i\in\{1,2\},
\]
step size \(\alpha\in(0,1)\), and discount factor \(\gamma=0.9\). With the ordering
\[
(1,1),(2,1),(1,2),(2,2),
\]
choose
\[
e_0=(c,-c,c,-c),
  \qquad c>0.
\]
The deterministic error remains on the one-dimensional symmetric
trajectory
\[
e_k=(x_k,-x_k,x_k,-x_k).
\]
To verify this, suppose the identity holds at time \(k\). In state \(1\),
\[
Q_k(1,1)=x_k,
  \qquad
  Q_k(1,2)=-1+x_k,
\]
and hence \(V_{Q_k}(1)-V^*(1)=x_k\). In state \(2\),
\[
Q_k(2,1)=-x_k,
  \qquad
  Q_k(2,2)=-1-x_k,
\]
and hence \(V_{Q_k}(2)-V^*(2)=-x_k\). Therefore, for each
\(i\in\{1,2\}\),
\[
\begin{aligned}
  e_{k+1}(1,i)
  &=
  x_k+\frac{\alpha}{4}(\gamma x_k-x_k)\\
  &=
  \left(1-\frac{\alpha(1-\gamma)}{4}\right)x_k, \\
  e_{k+1}(2,i)
  &=
  -x_k+\frac{\alpha}{4}(-\gamma x_k+x_k)\\
  &=
  -\left(1-\frac{\alpha(1-\gamma)}{4}\right)x_k,
\end{aligned}
\]
for \(k\in\{0,1,2,\ldots\}\). Since \(\gamma=0.9\),
\[
1-\frac{\alpha(1-\gamma)}{4}=1-0.025\alpha.
\]
Therefore,
\[
\begin{gathered}
x_k=(1-0.025\alpha)^k c,\\
e_k=(1-0.025\alpha)^k(c,-c,c,-c).
\end{gathered}
\]
Consequently,
\[
\begin{gathered}
e_k^+=(1-0.025\alpha)^k(c,0,c,0),\\
e_k^-=(1-0.025\alpha)^k(0,c,0,c),
\end{gathered}
\]
and hence
\[
\|e_k^+\|_\infty
  =
  \|e_k^-\|_\infty
  =
  (1-0.025\alpha)^k c,
\]
while
\[
\|e_k^+\|_2
  =
  \|e_k^-\|_2
  =
  \sqrt{2}\,(1-0.025\alpha)^k c.
\]
For this trajectory, the Bellman-max residuals vanish. The realized positive
and negative sign components therefore decay with exactly the same factor. This
does not contradict the comparison results above; it shows only that the
certified sign-separated asymmetry is a property of the available comparison
envelopes, not a statement that every deterministic trajectory must exhibit
strictly faster realized decay of \(e_k^-\) than of \(e_k^+\).
\end{example}

\setcounter{theorem}{10}
\section{Stochastic Q-learning}
\label{subsec:stochastic-comparison-systems}

As in the deterministic analysis, the stochastic lower comparison system is an
LTI system, whereas the upper comparison system is a linear switching system. The
lower stochastic residual identity is
\begin{equation}
\label{eq:error-optimized-lower-form}
\begin{gathered}
e_{k+1}
  =
  \mathbf A_-^\star e_k
  +\alpha\gamma DP
  \bigl(V_{Q_k}-V^*-\boldsymbol{\Pi}^{\pi_-^\star}e_k\bigr)
  +\alpha w_k,\\
k\in\{0,1,2,\ldots\},
\end{gathered}
\end{equation}
where the residual satisfies
\(V_{Q_k}-V^*-\boldsymbol{\Pi}^{\pi_-^\star}e_k\ge0\). Accordingly, as in the
deterministic lower comparison, we keep the same stochastic increment
\(\alpha w_k\) and remove this nonnegative residual. This gives the optimized
lower comparison system
\begin{equation}
\label{eq:optimized-lower-comparison-system}
\ell_{k+1}
  =
  \mathbf A_-^\star\ell_k+\alpha w_k,
  \qquad
  \ell_0=e_0,
  \qquad k\in\{0,1,2,\ldots\}.
\end{equation}
The system in \cref{eq:optimized-lower-comparison-system} is a fixed-mode LTI
comparison system driven by the same noise as the original error recursion; the
nonnegative residual removed from \cref{eq:error-optimized-lower-form} is what
makes it a lower comparison.

For the upper comparison, let us define the predictable state-wise maximizer
\[
\pi_k^+(s)\in\Argmax_{a\in\calA}e_k(s,a),
  \qquad s\in\calS.
\]
The upper stochastic residual identity is
\begin{equation}
\label{eq:error-direct-upper-form}
\begin{gathered}
e_{k+1}
  =
  \mathbf A_{\pi_k^+}e_k
  -\alpha\gamma DP
  \bigl(\boldsymbol{\Pi}^{\pi_k^+}e_k-(V_{Q_k}-V^*)\bigr)
  +\alpha w_k,\\
k\in\{0,1,2,\ldots\},
\end{gathered}
\end{equation}
where \(\boldsymbol{\Pi}^{\pi_k^+}e_k-(V_{Q_k}-V^*)\ge0\). Removing the
subtracted nonnegative residual gives the direct upper comparison system
\begin{equation}
\label{eq:direct-upper-comparison-system}
\begin{gathered}
u_{k+1}
  =
  \mathbf A_{\pi_k^+}u_k+\alpha w_k,\\
u_0=e_0,
  \qquad k\in\{0,1,2,\ldots\}.
\end{gathered}
\end{equation}
Unlike the lower comparison, \cref{eq:direct-upper-comparison-system} is a
linear stochastic switching system, because the mode \(\pi_k^+\) can change
with the current error. The switching signal \(\pi_k^+\) is predictable with
respect to the one-step update from time \(k\) to time \(k+1\), because it is
\(\calF_k\)-measurable.
\begin{lemma}\label{lem:stochastic-lower-upper-comparisons}
Under the standing Q-learning assumptions, with \(\ell_k\) and \(u_k\) defined
in \cref{eq:optimized-lower-comparison-system,eq:direct-upper-comparison-system}, we have
\[
\ell_k\le e_k\le u_k,
  \qquad
  k\in\{0,1,2,\ldots\}.
\]
\end{lemma}

\begin{proof}
Subtracting \cref{eq:optimized-lower-comparison-system} from
\cref{eq:error-optimized-lower-form} yields
\[
e_{k+1}-\ell_{k+1}
  =
  \mathbf A_-^\star(e_k-\ell_k)+\alpha\gamma DP
  \bigl(V_{Q_k}-V^*-\boldsymbol{\Pi}^{\pi_-^\star}e_k\bigr).
\]
The noise cancels. Since \(e_0-\ell_0=0\), all factors are nonnegative, and the
residual is nonnegative, induction gives \(\ell_k\le e_k\).

Subtracting \cref{eq:error-direct-upper-form} from
\cref{eq:direct-upper-comparison-system} gives
\[
\begin{aligned}
u_{k+1}-e_{k+1}
&=
  \mathbf A_{\pi_k^+}(u_k-e_k)\\
&\quad+\alpha\gamma DP
  \bigl(\boldsymbol{\Pi}^{\pi_k^+}e_k-(V_{Q_k}-V^*)\bigr).
\end{aligned}
\]
The noise again cancels. Since \(u_0-e_0=0\), all factors are nonnegative, and
the residual is nonnegative, induction gives \(e_k\le u_k\).
\end{proof}

The order relations imply
\[
e_k^-\le \ell_k^-,
  \qquad
  e_k^+\le u_k^+,
  \qquad k\in\{0,1,2,\ldots\},
\]
and the following inequalities hold:
\[
e_{k+1}^-
  \le
  \mathbf A_-^\star e_k^-
  +
  \alpha w_k^-,
  \qquad k\in\{0,1,2,\ldots\},
\]
\[
e_{k+1}^+
  \le
  \mathbf A_{\pi_k^+}e_k^+
  +
  \alpha w_k^+,
  \qquad k\in\{0,1,2,\ldots\}.
\]
The formal sign-control and one-step domination lemmas are stated and proved in
Appendix~\ref{app:stochastic-comparison-proofs}.

These one-step inequalities define the stochastic sign comparison systems used
in the finite-time bounds. The negative sign comparison keeps only the
fixed-mode LTI propagation of the previous negative part and the negative part
of the noise increment:
\begin{equation}
\label{eq:negative-comparison-system}
\begin{aligned}
z^-_{k+1}
&=
  \mathbf A_-^\star z^-_k
  +
  \alpha w_k^-,\\
z^-_0&=e_0^-,
  \qquad k\in\{0,1,2,\ldots\}.
\end{aligned}
\end{equation}
The positive sign comparison keeps the switching propagation selected by
\(\pi_k^+\) and the positive part of the noise increment:
\begin{equation}
\label{eq:positive-comparison-system}
\begin{aligned}
z^+_{k+1}
&=
  \mathbf A_{\pi_k^+}z^+_k
  +
  \alpha w_k^+,\\
z^+_0&=e_0^+,
  \qquad k\in\{0,1,2,\ldots\}.
\end{aligned}
\end{equation}
The next lemma states that these two auxiliary systems dominate the actual
negative and positive parts pathwise.

\begin{lemma}\label{lem:stochastic-pathwise-sign-comparison}
Under the standing Q-learning assumptions, with \(z_k^-\) and \(z_k^+\)
defined in \cref{eq:negative-comparison-system,eq:positive-comparison-system},
for all \(k\in\{0,1,2,\ldots\}\),
\[
e_k^-\le z^-_k,
  \qquad
  e_k^+\le z^+_k.
\]
\end{lemma}

\begin{proof}
At \(k=0\), the inequalities hold by definition. If they hold at time \(k\), then the one-step sign inequalities stated above and the nonnegativity of
\(\mathbf A_-^\star\) and \(\mathbf A_{\pi_k^+}\) give
\[
e_{k+1}^-
  \le
  \mathbf A_-^\star z_k^-+\alpha w_k^-
  =
  z_{k+1}^-,
\]
and
\[
e_{k+1}^+
  \le
  \mathbf A_{\pi_k^+}z_k^++\alpha w_k^+
  =
  z_{k+1}^+.
\]
The result follows by induction.
\end{proof}

\setcounter{theorem}{24}
\subsection{Stochastic comparison lemmas}\label{app:stochastic-comparison-proofs}

The first stochastic comparison lemma adds the Q-learning noise term to the deterministic residual identities.

\begin{lemma}\label{lem:stochastic-lower-upper-identities}
Under the standing Q-learning assumptions, the Q-learning error satisfies
\cref{eq:error-optimized-lower-form,eq:error-direct-upper-form}, with the
residual inequalities stated there.
\end{lemma}

\begin{proof}[Proof of \cref{lem:stochastic-lower-upper-identities}]
The deterministic identities in
\cref{lem:deterministic-lower-error-identity,lem:deterministic-upper-error-identity}
gain only the additive term \(\alpha w_k\) in the stochastic recursion. This
proves \cref{eq:error-optimized-lower-form,eq:error-direct-upper-form}.
\end{proof}

The next lemma converts the stochastic order comparisons into sign-part inequalities.

\begin{lemma}\label{lem:stochastic-sign-controlled-by-comparisons}
Under the standing Q-learning assumptions, with \(\ell_k\) and \(u_k\) defined
in \cref{eq:optimized-lower-comparison-system,eq:direct-upper-comparison-system},
\[
e_k^-\le \ell_k^-,
  \qquad
  e_k^+\le u_k^+,
  \qquad k\in\{0,1,2,\ldots\}.
\]
\end{lemma}

\begin{proof}[Proof of \cref{lem:stochastic-sign-controlled-by-comparisons}]
Since \(\ell_k\le e_k\), we have \(-e_k\le-\ell_k\), and hence
\(e_k^-=(-e_k)^+\le(-\ell_k)^+=\ell_k^-\). Since \(e_k\le u_k\), monotonicity
of the componentwise positive-part map gives \(e_k^+\le u_k^+\).
\end{proof}

The following one-step estimate separates the deterministic propagation from the signed noise increments.

\begin{lemma}\label{lem:sign-separated-one-step}
Under the standing Q-learning assumptions, the negative part satisfies
\[
e_{k+1}^-
  \le
  \mathbf A_-^\star e_k^-
  +
  \alpha w_k^-,
  \qquad k\in\{0,1,2,\ldots\}.
\]
The positive part satisfies
\[
e_{k+1}^+
  \le
  \mathbf A_{\pi_k^+}e_k^+
  +
  \alpha w_k^+,
  \qquad k\in\{0,1,2,\ldots\}.
\]
\end{lemma}

\begin{proof}[Proof of \cref{lem:sign-separated-one-step}]
From \cref{eq:error-optimized-lower-form} and \(e_k=e_k^+-e_k^-\),
\[
\begin{aligned}
e_{k+1}
&=
  \mathbf A_-^\star e_k^+
  -
  \mathbf A_-^\star e_k^-
  +
  r_k^-
  +
  \alpha w_k,\\
r_k^-&:=\alpha\gamma DP
  \bigl(V_{Q_k}-V^*-\boldsymbol{\Pi}^{\pi_-^\star}e_k\bigr)\ge0.
\end{aligned}
\]
The nonnegative vector \(\mathbf A_-^\star e_k^+ + r_k^-\) cannot increase the
negative part. Componentwise,
\[
e_{k+1}^-
  \le
  \left(\mathbf A_-^\star e_k^- -\alpha w_k\right)^+.
\]
Since \(\mathbf A_-^\star e_k^-\ge0\), the scalar inequality
\((a-\xi)^+\le a+\xi^-\) for \(a\ge0\), applied with
\(a=(\mathbf A_-^\star e_k^-)_i\) and \(\xi=\alpha w_{k,i}\), gives
\[
e_{k+1}^-
  \le
  \mathbf A_-^\star e_k^-+
  \alpha w_k^-.
\]
Similarly, from \cref{eq:error-direct-upper-form},
\[
\begin{aligned}
e_{k+1}
&=
  \mathbf A_{\pi_k^+}e_k^+
  -
  \mathbf A_{\pi_k^+}e_k^-
  -
  r_k^+
  +
  \alpha w_k,\\
r_k^+&:=\alpha\gamma DP
  \bigl(\boldsymbol{\Pi}^{\pi_k^+}e_k-(V_{Q_k}-V^*)\bigr)\ge0.
\end{aligned}
\]
The nonnegative vector \(\mathbf A_{\pi_k^+}e_k^-+r_k^+\) cannot increase the
positive part; hence
\[
e_{k+1}^+
  \le
  \left(\mathbf A_{\pi_k^+}e_k^+ +\alpha w_k\right)^+.
\]
Since \(\mathbf A_{\pi_k^+}e_k^+\ge0\), the scalar inequality
\((a+\xi)^+\le a+\xi^+\) for \(a\ge0\), applied componentwise, proves
\[
e_{k+1}^+
  \le
  \mathbf A_{\pi_k^+}e_k^+ +
  \alpha w_k^+.
\]
\end{proof}

\section{Restated Stochastic Finite-Time Bounds}\label{app:stochastic-rate-restatements}

\setcounter{theorem}{18}
\subsection{Noise moment bounds}\label{app:noise-proofs}

We begin the noise analysis with the boundedness and martingale-difference estimates needed later.

\begin{lemma}\label{lem:noise-moment-bound}
Under \cref{assump:basic}, let
\[
R_{\max}:=\max_{s,a,s'} |r(s,a,s')|,
\]
and define
\[
\begin{aligned}
W_{\max}:={}&
  \left(1+\sqrt{|\calS|\,|\calA|}\right)^2\\
&\times
  \left(
  R_{\max}+(1+\gamma)\right.\\
&\qquad\left.
  \max\bigl\{\|Q_0\|_\infty,R_{\max}/(1-\gamma)\bigr\}
  \right)^2.
\end{aligned}
\]
Then, for every \(k\in\{0,1,2,\ldots\}\),
\[
\|Q_k\|_\infty\le
  \max\bigl\{\|Q_0\|_\infty,R_{\max}/(1-\gamma)\bigr\},
\]
and the increment \(w_k\) is \(\calF_{k+1}\)-measurable and satisfies
\[
\E[w_k\mid\calF_k]=0,
  \qquad
  \E[\|w_k\|_2^2\mid\calF_k]\le W_{\max}.
\]
\end{lemma}

\begin{proof}[Proof of \cref{lem:noise-moment-bound}]
Since the state and action spaces are finite, \(R_{\max}<\infty\). We first
prove the pathwise bound on \(Q_k\). Suppose
\[
\|Q_k\|_\infty\le
  \max\bigl\{\|Q_0\|_\infty,R_{\max}/(1-\gamma)\bigr\}.
\]
If a coordinate is not sampled, it remains unchanged. If \((s_k,a_k)\) is sampled,
then
\[
\begin{aligned}
|Q_{k+1}(s_k,a_k)|
&\le
(1-\alpha)|Q_k(s_k,a_k)|\\
&\quad+
\alpha\left(|r_{k+1}|+\gamma\max_u |Q_k(s'_k,u)|\right) \\
&\le
(1-\alpha)\max\bigl\{\|Q_0\|_\infty,R_{\max}/(1-\gamma)\bigr\}\\
&\quad+
\alpha\left(R_{\max}+\gamma\right.\\
&\qquad\left.
\max\bigl\{\|Q_0\|_\infty,R_{\max}/(1-\gamma)\bigr\}\right).
\end{aligned}
\]
The last expression is at most
\[
\max\bigl\{\|Q_0\|_\infty,R_{\max}/(1-\gamma)\bigr\}
\]
because
\[
\begin{aligned}
&R_{\max}+\gamma
\max\bigl\{\|Q_0\|_\infty,R_{\max}/(1-\gamma)\bigr\}\\
&\quad\le
\max\bigl\{\|Q_0\|_\infty,R_{\max}/(1-\gamma)\bigr\}.
\end{aligned}
\]
Thus \(|Q_{k+1}(s_k,a_k)|\) satisfies the same bound. Since the other
coordinates remain bounded by the same quantity, induction gives
\[
\|Q_k\|_\infty\le
\max\bigl\{\|Q_0\|_\infty,R_{\max}/(1-\gamma)\bigr\},
\qquad k\ge0.
\]

The sample temporal-difference term satisfies
\[
\begin{aligned}
&\left|
  r_{k+1}+\gamma\max_{u\in\calA}Q_k(s'_k,u)-Q_k(s_k,a_k)
  \right|\\
&\quad\le
  R_{\max}+(1+\gamma)
  \max\bigl\{\|Q_0\|_\infty,R_{\max}/(1-\gamma)\bigr\}.
\end{aligned}
\]
Also, for every coordinate,
\[
\begin{aligned}
&\left|
R(s,a)+\gamma\sum_{s'}P(s'\mid s,a)V_{Q_k}(s')-Q_k(s,a)
\right|\\
&\quad\le R_{\max}+(1+\gamma)
\max\bigl\{\|Q_0\|_\infty,R_{\max}/(1-\gamma)\bigr\}.
\end{aligned}
\]
Because \(0<d(s,a)\le1\), this implies
\[
\begin{aligned}
\left\|D(F(Q_k)-Q_k)\right\|_2
&\le
  \sqrt{|\calS|\,|\calA|}\\
&\quad\times
  \left(R_{\max}+(1+\gamma)\right.\\
&\quad\left.
  \max\bigl\{\|Q_0\|_\infty,R_{\max}/(1-\gamma)\bigr\}\right).
\end{aligned}
\]
The sampled vector in \cref{eq:wk-def} has Euclidean norm at most
\(R_{\max}+(1+\gamma)\max\bigl\{\|Q_0\|_\infty,R_{\max}/(1-\gamma)\bigr\}\).
Hence, pathwise,
\[
\begin{aligned}
\|w_k\|_2
&\le
  \left(1+\sqrt{|\calS|\,|\calA|}\right)
  \left(R_{\max}+(1+\gamma)\right.\\
&\qquad\left.
  \max\bigl\{\|Q_0\|_\infty,R_{\max}/(1-\gamma)\bigr\}\right),
\end{aligned}
\]
and therefore
\[
\begin{aligned}
\E[\|w_k\|_2^2\mid\calF_k]
&\le
  \left(1+\sqrt{|\calS|\,|\calA|}\right)^2
  \left(R_{\max}+(1+\gamma)\right.\\
&\qquad\left.
  \max\bigl\{\|Q_0\|_\infty,R_{\max}/(1-\gamma)\bigr\}\right)^2\\
&=W_{\max}.
\end{aligned}
\]
The \(\calF_{k+1}\)-measurability follows from the fact that the fresh
observation is revealed between times \(k\) and \(k+1\), while
\(D(F(Q_k)-Q_k)\) is \(\calF_k\)-measurable.
Finally, using conditional independence of the fresh observation from \(\calF_k\),
\[
\begin{aligned}
&\E\!\left[
  e_{s_k,a_k}
  \left(
  \begin{aligned}
  &r_{k+1}+\gamma\max_{u\in\calA}Q_k(s'_k,u)\\
  &\quad-Q_k(s_k,a_k)
  \end{aligned}
  \right)
  \middle|\calF_k\right] \\
&=\sum_{s\in\calS}\sum_{a\in\calA}
  d(s,a)e_{s,a}\\
&\quad\times
  \Bigl(
  R(s,a)+\gamma\sum_{s'\in\calS}P(s'\mid s,a)V_{Q_k}(s')\\
&\qquad\qquad
  -Q_k(s,a)
  \Bigr) \\
&=D(F(Q_k)-Q_k).
\end{aligned}
\]
Together with the definition of \(w_k\), this gives
\(\E[w_k\mid\calF_k]=0\).
\end{proof}

The next lemma transfers the same second-moment control to the positive and negative noise components.

\begin{lemma}\label{lem:noise-sign-moments}
Under the conditions of \cref{lem:noise-moment-bound}, for every
\(k\in\{0,1,2,\ldots\}\),
\[
\E[\|w_k^-\|_2^2\mid\calF_k]\le W_{\max},
  \qquad
  \E[\|w_k^+\|_2^2\mid\calF_k]\le W_{\max}.
\]
\end{lemma}

\begin{proof}[Proof of \cref{lem:noise-sign-moments}]
Componentwise,
\[
0\le (w_{k,i}^-)^2\le w_{k,i}^2,
  \qquad
  0\le (w_{k,i}^+)^2\le w_{k,i}^2.
\]
Indeed, if \(w_{k,i}\ge0\), then \(w_{k,i}^-=0\) and
\((w_{k,i}^+)^2=w_{k,i}^2\). If \(w_{k,i}<0\), then
\((w_{k,i}^-)^2=w_{k,i}^2\) and \(w_{k,i}^+=0\). Hence each signed-part square
is either zero or exactly \(w_{k,i}^2\). Summing over coordinates gives
\(\|w_k^-\|_2^2\le\|w_k\|_2^2\) and
\(\|w_k^+\|_2^2\le\|w_k\|_2^2\). Taking conditional expectations and applying
\cref{lem:noise-moment-bound} proves the claim.
\end{proof}

\subsection{Stochastic finite-time bounds}\label{app:stochastic-rate-proofs}

For convenience, we restate the main stochastic finite-time bounds from
\cref{sec:finite-time-rates}. The proofs use reference filters driven by the
martingale difference \(w_k\).

\begin{lemma}\label{lem:reference-filter-noise-bound}
Let \(B\in\calM_\alpha\) be fixed, and let
\[
b_{k+1}=Bb_k+\alpha w_k,
  \qquad b_0=0.
\]
Then, for every \(k\ge0\),
\[
\E[\|b_k\|_\infty]
  \le
  |\calS|\,|\calA|
  \sqrt{\frac{\alpha W_{\max}}{d_{\min}(1-\gamma)}}.
\]
\end{lemma}

\begin{proof}
Set \(n:=|\calS|\,|\calA|\) and
\[
\rho_{\mathrm{row}}:=1-\alpha d_{\min}(1-\gamma).
\]
For every \(\pi\in\Theta\), the matrix \(\mathbf A_\pi\) is nonnegative and its
row sums are bounded by \(\rho_{\mathrm{row}}\). Hence
\(\|B^t\|_\infty\le \rho_{\mathrm{row}}^t\) for all \(t\ge0\). Since \(w_k\) is a
martingale difference and \(B\) is deterministic, the cross terms vanish in the
second-moment expansion of
\[
b_k=\alpha\sum_{i=0}^{k-1}B^{k-1-i}w_i.
\]
For each coordinate \(j\),
\[
\E[b_{k,j}^2]
  \le
  \alpha^2 W_{\max}
  \sum_{r=0}^{k-1}\rho_{\mathrm{row}}^{2r}
  \le
  \frac{\alpha^2 W_{\max}}{1-\rho_{\mathrm{row}}^2}.
\]
Because
\[
1-\rho_{\mathrm{row}}^2
  \ge
  \alpha d_{\min}(1-\gamma),
\]
we get
\[
\E[|b_{k,j}|]
  \le
  \sqrt{\frac{\alpha W_{\max}}{d_{\min}(1-\gamma)}}.
\]
Finally,
\[
\E[\|b_k\|_\infty]
  \le
  \sum_{j=1}^n\E[|b_{k,j}|]
  \le
  n\sqrt{\frac{\alpha W_{\max}}{d_{\min}(1-\gamma)}}.
\]
This proves the claim.
\end{proof}

\begin{restatementbox}
\noindent\textbf{Theorem~\ref{thm:negative-part-rate} (negative-part bound).}
Assume \cref{assump:basic}, and let \(W_{\max}\) be defined in
\cref{eq:wmax-def}. Fix \(\varepsilon>0\) such that
\(\rho_-^\star+\varepsilon<1\), and define
\[
\beta_-:=\rho_-^\star+\varepsilon,
  \qquad
K_-:=K_{\beta_-}(\{\mathbf A_-^\star\}).
\]
Then, for every \(k\ge0\),
\[
\E[\|e_k^-\|_\infty]
  \le
  K_-\beta_-^k\|e_0\|_\infty
  +
  |\calS|\,|\calA|
  \sqrt{\frac{\alpha W_{\max}}{d_{\min}(1-\gamma)}}.
\]
\end{restatementbox}

\begin{proof}[Proof of \cref{thm:negative-part-rate}]
By \cref{lem:stochastic-sign-controlled-by-comparisons},
\[
\|e_k^-\|_\infty\le \|\ell_k\|_\infty,
\]
where \(\ell_{k+1}=\mathbf A_-^\star\ell_k+\alpha w_k\) and \(\ell_0=e_0\).
Split \(\ell_k=a_k+b_k\), with
\[
\begin{gathered}
a_{k+1}=\mathbf A_-^\star a_k,
  \qquad a_0=e_0,\\
b_{k+1}=\mathbf A_-^\star b_k+\alpha w_k,
  \qquad b_0=0.
\end{gathered}
\]
The product-growth estimate gives
\[
\|a_k\|_\infty
  \le
  K_-\beta_-^k\|e_0\|_\infty.
\]
By \cref{lem:reference-filter-noise-bound},
\[
\E[\|b_k\|_\infty]
  \le
  |\calS|\,|\calA|
  \sqrt{\frac{\alpha W_{\max}}{d_{\min}(1-\gamma)}}.
\]
The triangle inequality completes the proof.
\end{proof}

\begin{restatementbox}
\noindent\textbf{Corollary~\ref{cor:orthant-distance-rate} (orthant-distance bound).}
Under the assumptions of Theorem~\ref{thm:negative-part-rate}, for every
\(\varepsilon>0\) satisfying \(\rho_-^\star+\varepsilon<1\), with
\(\beta_-:=\rho_-^\star+\varepsilon\) and
\(K_-:=K_{\beta_-}(\{\mathbf A_-^\star\})\), every \(k\ge0\) satisfies
\[
\begin{aligned}
\E\!\left[
  \dist_\infty(e_k,\R_+^{|\calS|\,|\calA|})
  \right]
&\le
  K_-\beta_-^k\|e_0\|_\infty\\
&\quad+
  |\calS|\,|\calA|
  \sqrt{\frac{\alpha W_{\max}}{d_{\min}(1-\gamma)}}.
\end{aligned}
\]
\end{restatementbox}

\begin{proof}[Proof of \cref{cor:orthant-distance-rate}]
For every vector \(x\),
\[
\dist_\infty(x,\R_+^{|\calS|\,|\calA|})=\|x^-\|_\infty.
\]
Applying this identity with \(x=e_k\) and then using
\cref{thm:negative-part-rate} proves the claim.
\end{proof}

\begin{restatementbox}
\noindent\textbf{Theorem~\ref{thm:positive-part-rate} (positive-part bound).}
Assume \cref{assump:basic}, and let \(W_{\max}\) be defined in
\cref{eq:wmax-def}. Fix \(\varepsilon>0\) such that
\(\rho_++\varepsilon<1\), and define
\[
\beta_+:=\rho_++\varepsilon,
  \qquad
K_+:=K_{\beta_+}(\calM_\alpha).
\]
Then, for every \(k\ge0\),
\[
\begin{aligned}
\E[\|e_k^+\|_\infty]
&\le
  K_+\beta_+^k\|e_0\|_\infty\\
&\quad+
  2\alpha\gamma d_{\max}K_+^2
  k\beta_+^{k-1}\|e_0\|_\infty\\
&\quad+
  \left(1+\frac{2\gamma d_{\max}}{d_{\min}(1-\gamma)}\right)\\
&\qquad\times
  |\calS|\,|\calA|
  \sqrt{\frac{\alpha W_{\max}}{d_{\min}(1-\gamma)}}.
\end{aligned}
\]
Here \(k\beta_+^{k-1}\) is interpreted as zero when \(k=0\).
\end{restatementbox}

\begin{proof}[Proof of \cref{thm:positive-part-rate}]
By \cref{lem:stochastic-sign-controlled-by-comparisons},
\[
\|e_k^+\|_\infty\le \|u_k\|_\infty,
\]
where \(u_{k+1}=\mathbf A_{\pi_k^+}u_k+\alpha w_k\) and \(u_0=e_0\). Fix any
\(\bar\pi\in\Theta\), and introduce the reference filter
\[
x_{k+1}=\mathbf A_{\bar\pi}x_k+\alpha w_k,
  \qquad x_0=e_0.
\]
Split \(x_k=a_k+b_k\), where
\[
\begin{gathered}
a_{k+1}=\mathbf A_{\bar\pi}a_k,
  \qquad a_0=e_0,\\
b_{k+1}=\mathbf A_{\bar\pi}b_k+\alpha w_k,
  \qquad b_0=0.
\end{gathered}
\]
The product-growth estimate gives
\[
\|a_k\|_\infty
  \le
  K_+\beta_+^k\|e_0\|_\infty,
\]
and \cref{lem:reference-filter-noise-bound} gives
\[
\E[\|b_k\|_\infty]
  \le
  |\calS|\,|\calA|
  \sqrt{\frac{\alpha W_{\max}}{d_{\min}(1-\gamma)}}.
\]

Let \(r_k:=u_k-x_k\). Then \(r_0=0\) and
\[
r_{k+1}
  =
  \mathbf A_{\pi_k^+}r_k
  +
  (\mathbf A_{\pi_k^+}-\mathbf A_{\bar\pi})x_k.
\]
For every \(\pi\in\Theta\),
\[
\|\mathbf A_{\pi}-\mathbf A_{\bar\pi}\|_\infty
  \le
  2\alpha\gamma d_{\max},
\]
because two deterministic policy-selection rows differ by at most two in
\(\ell_1\)-norm. Split \(r_k=c_k+d_k\) into the parts driven by \(a_k\) and
\(b_k\), respectively. The product-growth estimate yields
\[
\begin{aligned}
\|c_k\|_\infty
&\le
  \sum_{i=0}^{k-1}
  K_+\beta_+^{k-1-i}
  2\alpha\gamma d_{\max}
  K_+\beta_+^i\|e_0\|_\infty\\
&=
  2\alpha\gamma d_{\max}K_+^2
  k\beta_+^{k-1}\|e_0\|_\infty.
\end{aligned}
\]
For the noise-driven residual, use the row-sum contraction
\(\|\mathbf A_{\pi}\|_\infty\le \rho_{\mathrm{row}}\) with
\(\rho_{\mathrm{row}}=1-\alpha d_{\min}(1-\gamma)\). Then
\[
\begin{aligned}
\E[\|d_k\|_\infty]
&\le
  \sum_{i=0}^{k-1}
  \rho_{\mathrm{row}}^{k-1-i}
  2\alpha\gamma d_{\max}
  \E[\|b_i\|_\infty]\\
&\le
  \frac{2\gamma d_{\max}}{d_{\min}(1-\gamma)}
  |\calS|\,|\calA|
  \sqrt{\frac{\alpha W_{\max}}{d_{\min}(1-\gamma)}}.
\end{aligned}
\]
Combining
\[
\|u_k\|_\infty\le \|a_k\|_\infty+\|b_k\|_\infty+\|c_k\|_\infty+\|d_k\|_\infty
\]
with the estimates above proves the theorem.
\end{proof}

\end{document}